\documentclass{article}
 \pdfoutput=1
\usepackage{etex}
\PassOptionsToPackage{numbers, compress}{natbib}
\usepackage[final]{nips_2016}

\usepackage[utf8]{inputenc} 
\usepackage[T1]{fontenc}    
\usepackage{url}            
\usepackage{booktabs}       
\usepackage{amsfonts}       
\usepackage{nicefrac}       
\usepackage{microtype}      

\title{Clustering Signed Networks with the \\Geometric Mean of Laplacians}

  \author{
  Pedro Mercado$^1$, Francesco Tudisco$^2$ and Matthias Hein$^1$ \\
  $^1$Saarland University, Saarbr\"{u}cken, Germany\\
  $^2$University of Padua, Padua, Italy\\
}

\usepackage[pdftex]{graphicx}
\usepackage{epstopdf}

\usepackage{times}
\usepackage{subfigure} 
\usepackage{sidecap}

\usepackage{natbib}

\usepackage{algorithmic}
\usepackage[plain,linesnumbered]{algorithm2e}

\usepackage{ctable} 

 \newcommand{\myComment}[1]{#1} \def\nosupplementary#1{}

\usepackage{multirow}
\usepackage{amsthm}
\usepackage{comment}
\usepackage{amsopn}
\usepackage{amssymb}
\usepackage{booktabs}
\usepackage{amsmath}
\usepackage{mathrsfs}
\usepackage{bbm}
\usepackage{bm}
\usepackage{enumitem}
\usepackage{caption}
\usepackage{wrapfig}
\usepackage{floatrow}  

\usepackage{multicol}

\def\Pr{\mathrm{P}}

\def\R{\mathbb{R}}

\newcommand{\abs}[1]{\left|#1\right|}

\def\Pr{\mathrm{P}}

\def\vol{\mathop{\textrm vol}\nolimits}

\def\diag{\mathrm{diag}}
\def\one{\mathbf{1}}

\def\pp{p_{\mathrm{in}}^{+}}
\def\qp{p_{\mathrm{out}}^{+}}

\def\ppm{p_{\mathrm{in}}^{-}}
\def\qm{p_{\mathrm{out}}^{-}}

\def\x{\mathbf x} \def\y{\mathbf y}  
\def\u{\mathbf u} \def\v{\mathbf v}  
 \def\e{\mathbf e}

\def\sym{\mathrm{sym}}
\def\bal{\mathrm{bal}}
\def\vol{\mathrm{vol}}
\def\conf{\mathrm{conf}}

\def\fra#1{\textcolor{black}{#1}} 
\def\pedro#1{\textcolor{black}{#1}} 

\def\pedronew#1{\textcolor{black}{#1}}

\usepackage[frame,arrow,curve]{xy}

\newcommand{\arcodot}[3][]{\ar@{.}"#2";"#3"#1}
\newcommand{\arcoD}[3][]{\ar@{-} "#2";"#3"#1}
\newcommand{\arcoO}[3][]{\ar@/^/@{--}"#2";"#3"#1}
\newcommand{\arcoA}[3][]{\ar@/_/@{--}"#2";"#3"#1}
\newcommand{\arcoPlusD}[3][]{\ar@{-}|\oplus "#2";"#3"#1}
\newcommand{\arcoMinusO}[3][]{\ar@/^/@{--}|\ominus"#2";"#3"#1}
\newcommand{\arcoMinusA}[3][]{\ar@/_/@{--}|\ominus"#2";"#3"#1}
\newcommand{\ArcoD}[3][]{\ar@{->}"#2";"#3"#1}
\newcommand{\ArcoO}[3][]{\ar@/^/@{-->}"#2";"#3"#1}
\newcommand{\ArcoA}[3][]{\ar@/_/@{-->}"#2";"#3"#1}
\newcommand{\loopU}[2][]{\ar@(ur,ul)"#2";"#2"#1}
\newcommand{\loopD}[2][]{\ar@(dr,dl)"#2";"#2"#1}

\def\old#1{}

\newtheorem{thm}{Theorem}
\newtheorem{mydef}{Definition}

\newtheorem{corollary}{Corollary}

\begin{document}

\maketitle

\begin{abstract} 

Signed networks allow to model positive  and negative  relationships. 
We analyze existing extensions of spectral clustering to signed networks. 
It turns out that existing approaches do not recover the ground truth clustering in several situations where either the positive or the negative network structures contain no noise. Our analysis shows that these
problems arise as existing approaches take some form of arithmetic mean of the Laplacians of the positive
and negative part. As a solution we propose to use the geometric mean of the Laplacians of positive and
negative part and show that it outperforms the existing approaches. 
While the geometric mean of matrices is computationally expensive, we show that eigenvectors of the geometric mean can be computed efficiently, 
leading to a 
numerical scheme for sparse matrices which is of independent interest.
\end{abstract} 

\section{Introduction}\label{sec:Introduction}
A signed graph is a graph with positive and negative edge weights. 
Typically  positive edges model attractive relationships between objects 
such as similarity or friendship and negative edges model repelling relationships such as dissimilarity or enmity.  
The concept  of balanced signed networks can be traced back to \cite{Harary:1954:Notion, Cartwright:1956:Structural}. Later, in \cite{Davis:1967:Clustering},  a signed graph  is defined as $k$-balanced if there exists a partition into $k$ groups where 
only positive edges are within the groups and negative edges are between the groups. 
Several approaches to find communities in signed graphs have been proposed (see~\cite{tang2015survey} for an overview).
In this paper we focus on extensions of spectral clustering to signed graphs.
Spectral clustering is a well established method for unsigned graphs which, based on the first eigenvectors of the graph Laplacian, embeds nodes of the graphs in $\R^k$  and then uses $k$-means to find the partition. In \cite{Kunegis:2010:spectral} the idea is transferred to signed graphs. 
They define \pedronew{the} signed ratio and normalized cut functions and 
show that the spectrum of suitable signed graph Laplacians \pedronew{yield}  a relaxation of those objectives.
In~\cite{Chiang:2012:Scalable} other objective functions for  signed graphs are introduced. 
They show that a relaxation of their objectives is equivalent to weighted kernel $k$-means by choosing an appropriate kernel. 
While they have a scalable method for clustering, they report that they can not find any cluster structure in real
world signed networks.

We show that the existing extensions of the graph Laplacian to signed graphs used for spectral clustering have
severe deficiencies. Our analysis of the stochastic block model for signed graphs shows that, even for the perfectly balanced 
case, recovery of the ground-truth clusters is not guaranteed. The reason is that the eigenvectors encoding the cluster 
structure do not necessarily correspond to the smallest eigenvalues, thus leading to a noisy embedding of the data points 
and in turn failure of $k$-means to recover the cluster structure. The implicit mathematical reason is that all 
existing extensions of the graph Laplacian are based on some form of arithmetic mean of 
operators of the positive and negative graphs. 
In this paper we suggest as a solution to use the  geometric mean of the Laplacians of 
positive and negative part. In particular, we show that in the stochastic block model the geometric 
mean Laplacian allows in expectation to recover the ground-truth clusters in any reasonable clustering setting. 
A main challenge for our approach is that the geometric mean Laplacian is 
computationally expensive and does not scale to large sparse networks. 
Thus a main contribution of this paper is showing that the first few eigenvectors of the geometric 
mean can still be computed efficiently.  Our algorithm is based on the inverse power method and the 
extended Krylov subspace technique introduced by \cite{druskin:1998:extendedKrylov} and allows to 
compute eigenvectors of the  geometric mean $A\# B$ of two matrices $A,B$ without ever computing $A\# B$ itself. 

In Section \ref{sec:signed_graph_clustering} we discuss existing work on Laplacians on signed graphs. 
In Section \ref{sec:GeoMean} we discuss the geometric mean of two  matrices and introduce the 
geometric mean Laplacian which is the basis of our spectral clustering method for signed graphs. 
In Section \ref{sec:stochastic_block_model} we analyze our and existing approaches for the stochastic block model. 
In Section \ref{sec:IPM} we introduce our efficient algorithm to compute eigenvectors of the geometric mean
of two matrices, and finally in Section \ref{sec:real_world_experiments} we discuss performance
of our approach on real world graphs. \nosupplementary{\textit{Proofs have been moved to the supplementary material}}.

\section{Signed graph clustering}\label{sec:signed_graph_clustering}
Networks encoding positive and negative relations among the nodes \pedronew{can be} represented by weighted signed graphs.
Consider two symmetric non-negative weight matrices $W^+$ and $W^-$, 
a vertex set $V=\{v_1, \dots, v_n\}$, and let  $G^+ =(V,W^+)$ and $G^- = (V,W^-)$ be the induced graphs. 
A signed graph  is the pair $G^\pm = (G^+, G^-)$ where $G^+$ and $G^-$ encode positive and the negative relations, respectively.

The concept of community in signed networks is typically related to the theory of social balance. This theory,  as presented in~\cite{Harary:1954:Notion,Cartwright:1956:Structural},  
is based on the analysis of affective ties, where positive ties are a source of balance whereas negative ties are considered as a source of imbalance in social groups. 
\begin{mydef}[\cite{Davis:1967:Clustering}, $k$-balance]\label{definition:kBalancedSignedGraph}
A signed graph is $k$-balanced if the set of vertices can 
be partitioned into $k$ sets such that within the subsets there are only positive edges,
 and between them only negative. 
\end{mydef}
The presence of $k$-balance in $G^\pm$  implies the presence of $k$ 
groups of nodes  being both assortative in $G^+$ and dissassortative in $G^-$. 
However this situation is fairly rare in  real world networks  and expecting communities in signed networks to 
be a perfectly balanced set of nodes is unrealistic.  

In the next section we will show that Laplacians inspired by Definition~\ref{definition:kBalancedSignedGraph}
are based on some form of arithmetic mean of Laplacians. 
As an alternative we propose the geometric mean of Laplacians and 
show that it is able to recover communities
when either $G^+$ is assortative, or $G^-$ is disassortative, or both.
Results of this paper will make clear that the use of the geometric mean of Laplacians allows to recognize communities where previous approaches fail.

\subsection{Laplacians on Unsigned Graphs}
Spectral clustering of undirected, unsigned graphs
using the Laplacian matrix is a well established technique (see~\cite{Luxburg:2007:tutorial} for an overview). 
Given an unsigned graph $G=(V,W)$, the Laplacian and its normalized version are defined as 
\begin{equation}\label{eq:Laplacian}
  L = D - W \quad\qquad\,\, L_\sym = D^{-1/2}LD^{-1/2}
\end{equation}
where $D_{ii} = \sum_{j=1}^n w_{ij}$ is the diagonal matrix of the degrees of $G$. Both Laplacians are
positive semidefinite, and the multiplicity $k$ of the eigenvalue $0$ is equal to the number of connected components in the graph.
Further, the Laplacian is suitable in assortative cases~\cite{Luxburg:2007:tutorial}, \textit{i.e.}\ for the identification of clusters under the assumption that 
the amount of edges inside clusters {has to be} larger than the amount of edges between {them}. 

For disassortative cases, \textit{i.e.}\ for the identification of clusters where the amount of edges {has to be} larger between clusters
than inside clusters, the signless Laplacian is a better choice \cite{Liu2015}. 
Given the unsigned graph $G=(V,W)$, the signless Laplacian  and its normalized version are defined as
\begin{equation}\label{eq:signlessLaplacian}
  Q = D + W, \quad\qquad\,\, Q_\sym = D^{-1/2}QD^{-1/2}
\end{equation}
Both Laplacians  are positive semi-definite, and the smallest eigenvalue is zero if and only if the graph has a bipartite component~\cite{Desai:1994:characterization}. 
\subsection{Laplacians on Signed Graphs}
Recently a number of Laplacian operators for signed networks have been introduced.  
Consider the signed graph $G^\pm = (G^+,G^-)$. 
Let $D_{ii}^+ = \sum_{j=1}^n w^+_{ij}$ be the diagonal matrix of 
the degrees of $G^+$ and $\bar{D}_{ii} = \sum_{j=1}^n w^+_{ij} + w^-_{ij}$ the one of the overall degrees in $G^\pm$.

The following Laplacians for signed networks have been considered so far
%
\begin{equation}\label{eq:signed_laplacians}
  \begin{aligned}
     L_{BR} &= D^+ - W^+ \pedro{+}W^- ,    \,\, L_{BN} = \bar{D}^{-1} L_{BR},                   \qquad\quad\quad \text{\footnotesize{(balance ratio/normalized Laplacian)}}\\
     L_{SR} &= \bar{D} - W^+ \pedro{+} W^-, \quad L_{SN} = \bar{D}^{-1/2} L_{SR} \bar{D}^{-1/2}, \quad \text{\footnotesize{(signed ratio/normalized Laplacian)}}
  \end{aligned}
\end{equation}
and spectral clustering algorithms have been proposed for $G^\pm$, based on these Laplacians \cite{Kunegis:2010:spectral, Chiang:2012:Scalable}. Let $L^+$ and $Q^-$ be the Laplacian and the signless Laplacian matrices of the graphs $G^+$ and $G^-$, respectively.  
We note that the matrix $L_{SR}$ blends the informations from $G^+$ and $G^-$ into (twice)
the arithmetic mean of $L^+$ and $Q^-$, namely the following identity holds
\begin{equation}\label{eq:L_SR}
 L_{SR} = L^+ + Q^- \, .
\end{equation}
Thus, as an alternative to the normalization defining $L_{SN}$ from $L_{SR}$, it is natural to consider the  arithmetic mean of the normalized Laplacians $L_{AM}= L^{+}_\sym + Q^{-}_\sym$.  
In the next section we introduce the geometric mean of $L_\sym^+$ and $Q_{\sym}^-$ and propose a new clustering algorithm for 
signed graphs based on that matrix. The analysis and experiments of next sections will show that blending the information 
from the positive and negative graphs trough the geometric mean overcomes the deficiencies showed by the arithmetic mean based operators.

\section{Geometric mean of Laplacians}\label{sec:GeoMean}
We define here the geometric mean of matrices and introduce the geometric mean of normalized Laplacians for clustering signed networks.  
Let $A^{1/2}$ be the unique positive definite solution of the matrix equation $X^2 = A$\pedro{, where $A$ is positive definite}.

\begin{mydef}\label{def:geometricMean}
 Let $A,B$ be positive definite matrices.
 The geometric mean of $A$ and $B$ is the positive definite matrix $A\# B$ defined by $ A\# B = A^{1/2} ( A^{-1/2} B A^{-1/2} )^{1/2} A^{1/2}$.
\end{mydef}
One can prove that $A\# B = B\#A$ (see \cite{bhatia2009positive} for details). Further, there are several useful ways to represent the geometric mean of positive definite matrices (see f.i.\ \cite{bhatia2009positive,Ianazzo:2012:geometricMean}) 
\begin{equation}\label{eq:geometricMean:Representations}
 \begin{aligned}
  A\# B = A(A^{-1}B)^{1/2} = (BA^{-1})^{1/2}A  =  B(B^{-1}A)^{1/2} = (AB^{-1})^{1/2}B
\end{aligned}
\end{equation}
The next result reveals  further consistency with the scalar case, in fact we observe that  if  $A$ and $B$ have some eigenvectors in common, 
then $A+B$ and $A\# B$ have those eigenvectors, with eigenvalues given by the arithmetic and geometric mean 
of the corresponding eigenvalues of $A$ and $B$, respectively. 
\begin{thm}\label{thm:AB-mEigenvectors}
Let $\u$ be an eigenvector of $A$ and $B$ with eigenvalues $\lambda$ and $\mu$, respectively. 
Then, $\u$ is an eigenvector of $A+B$ and $A\#B$ with eigenvalue 
$\lambda+\mu$ and $\sqrt{\lambda \mu}$, respectively.
\end{thm}
\myComment{
\begin{proof}
Using the identities $A\u=\lambda \u$ and $B\u=\mu \u$ we have $(A+B)\u = (\lambda + \mu )\u$. For the geometric mean, observe that for any positive definite matrix $M$, if $M\x = \lambda(M)\x$, then $M^{1/2}\x = \lambda(M)^{1/2}\x$. In particular we have
\begin{align*}
A^{-1/2} B A^{-1/2}\u = \lambda^{-1/2} A^{-1/2} B \u = \lambda^{-1/2}\mu A^{-1/2}\u = (\mu/\lambda) \u
\end{align*}
thus $(A^{-1/2} B A^{-1/2})^{1/2}\u  = \sqrt{\mu/\lambda}\, \u$. As a consequence
$$(A\#B)\u = A^{1/2} ( A^{-1/2} B A^{-1/2} )^{1/2} A^{1/2}\u = \lambda^{1/2}A^{1/2} ( A^{-1/2} B A^{-1/2} )^{1/2}\u = ( \sqrt{\lambda\mu})\u$$
which concludes the proof.
\end{proof}
}
\subsection{Geometric mean for signed networks clustering}\label{sec:GeoMeanLaplacian} 
Consider the signed network $G^\pm=(G^+,G^-)$. 
We define the normalized geometric mean Laplacian of $G^{\pm}$ as
\begin{equation}\label{eq:geomean_operator}
 L_{GM}=L^+_\sym \# Q^-_\sym\,
\end{equation}
\pedronew{W}e propose Algorithm~\ref{alg:geometricMeanClusteringSignedNetworks} for clustering signed networks, based on the spectrum of $L_{GM}$. 
\pedronew{
By definition \ref{def:geometricMean}, the matrix geometric mean $A\#B$ requires $A$ and $B$ to be positive definite.
As both the Laplacian and the signless Laplacian are positve semi-definte, 
\fra{in what follows we shall assume that the matrices $L^+_\sym$ and $Q^-_\sym$ in \eqref{eq:geomean_operator} are modified by a small diagonal shift, ensuring positive definiteness. That is, in practice, we consider $L^+_\sym + \varepsilon_1 I$ and $Q^-_\sym + \varepsilon_2 I$ being $\varepsilon_1$ and $\varepsilon_2$ small positive numbers. 
For the sake of brevity,  we do not explicitly write the shifting matrices.} 
%
}
\RestyleAlgo{plain}
  \rule{\linewidth}{1pt}
  \begin{minipage}{\textwidth}
 \begin{algorithm}[H]
  \DontPrintSemicolon
	\caption{Spectral clustering with $L_{GM}$ on signed  networks}\label{alg:geometricMeanClusteringSignedNetworks}
	\KwIn{Symmetric weight matrices $W^+, W^-\in\mathbb{R}^{n\times n}$, number $k$ of clusters to construct. }
	\KwOut{Clusters $C_1,\ldots,C_k$.}
	 Compute the $k$ eigenvectors $\u_1, \ldots,\u_k$ corresponding to the $k$ smallest eigenvalues of $L_{GM}$.\;
	 Let $U=(\u_1, \ldots,\u_k)$. \;
	 Cluster the rows of $U$ with $k$-means into clusters $C_1,\ldots,C_k$.	\;
\end{algorithm} 
  \end{minipage}
 \rule{\linewidth}{1pt}
\old{As it is standard in spectral clustering, we consider the normalized operator over the unnormalized version $L^+\#Q^-$.
In fact, extensive comparisons have been done between the two geometric mean Laplacians, showing that the proposed normalized operator $L_{GM}$ is superior. For this reason and for sake of brevity, we only report the behavior and the  analysis of $L_{\sym}^+\# Q_\sym^-$.}

The main bottleneck of  Algorithm \ref{alg:geometricMeanClusteringSignedNetworks} is the computation of the eigenvectors in step 1. 
In Section \ref{sec:IPM} we propose a scalable Krylov-based method to handle this problem.

Let us briefly discuss the motivating intuition behind the proposed clustering strategy. 
Algorithm \ref{alg:geometricMeanClusteringSignedNetworks}, as well as state-of-the-art clustering algorithms based on the matrices 
in \eqref{eq:signed_laplacians}, rely on the $k$ smallest  eigenvalues of the considered operator and their corresponding eigenvectors. 
Thus the relative ordering of the eigenvalues plays a crucial role.  
Assume the eigenvalues to be enumerated in ascending order. 
Theorem \ref{thm:AB-mEigenvectors} 
states that the functions $(A,B)\mapsto A+B$ and $(A, B)$ $\mapsto$ $A\# B$ map eigenvalues of $A$ and $B$ having the same corresponding eigenvectors, 
into the arithmetic mean $\lambda_i(A)+\lambda_j(B)$ and geometric mean
$\sqrt{\lambda_i(A)\lambda_j(B)}$, respectively\pedro{, where $\lambda_i(\cdot)$ is the $i^{th}$ smallest eigenvalue of the corresponding matrix}. Note that the indices $i$ and $j$ are not the same in general, 
as the eigenvectors shared by $A$ and $B$ may be associated to eigenvalues having different  positions in the relative ordering of $A$ and $B$.
\fra{This intuitively suggests } that small eigenvalues of $A+B$ are related to small eigenvalues of both $A$ and $B$, whereas those of $A\#B$ are associated with small eigenvalues of either $A$ or $B$, or both. Therefore the relative ordering of the small eigenvalues of $L_{GM}$ is influenced by the presence of assortative clusters in $G^+$ (related to small eigenvalues of $L_\sym^+$) or by disassortative clusters in $G^-$ (related to small eigenvalues in $Q_\sym^-$), whereas the ordering of the small eigenvalues of the arithmetic mean takes into account only the presence of both those situations.

In the next section, for networks following the stochastic block model, we analyze in expectation the spectrum of the normalized geometric mean Laplacian as well as the 
one of the normalized Laplacians previously introduced.
In this case the expected spectrum can be computed explicitly and we observe that in expectation the 
ordering induced by blending the informations of $G^+$ and $G^-$ trough the geometric mean allows 
to recover the ground truth clusters perfectly, whereas the use of the arithmetic mean introduces
a bias which reverberates into a significantly higher clustering error.

\section{Stochastic block model on signed graphs}\label{sec:stochastic_block_model}
In this section we present an analysis of different signed graph Laplacians based on the Stochastic Block Model (\textbf{SBM}). 
The SBM is a widespread benchmark generative model for networks showing a clustering, community, or group behaviour~\cite{rohe2011spectral}.
Given a prescribed set of groups of nodes, the SBM  defines the presence of an edge as a random variable with probability being dependent on which groups it joins. 
To our knowledge this is the first analysis of spectral clustering on signed graphs with the stochastic block model.
Let $\mathcal{C}_1, \ldots, \mathcal{C}_k$ be ground truth clusters, all having the same size $\abs{\mathcal{C}}$.
We let $\pp$ ($\ppm$) be the probability that there exists a positive (negative) edge between  nodes in the same cluster, and  let  $\qp$ ($\qm$) denote
 the probability of a positive (negative) edge between  nodes in  different clusters.
 
Calligraphic letters denote matrices in expectation.
In particular $\mathcal{W}^+$ and  $\mathcal{W}^-$ denote the weight matrices in expectation. We have 
$\mathcal{W}^+_{i,j}=\pp$ and $\mathcal{W}^-_{i,j}=\ppm$ if $v_i,v_j$ belong to the same cluster, whereas
$\mathcal{W}^+_{i,j}=\qp$ and $\mathcal{W}^-_{i,j}=\qm$ if $v_i,v_j$ belong to different  clusters.
\pedronew{Sorting} nodes according to the ground truth clustering shows that $\mathcal{W}^+$ and $\mathcal{W}^-$ have rank $k$. 

\begin{table}[t!]
\begin{minipage}{\linewidth}
 \rule{\linewidth}{1pt}
\begin{minipage}{.4\linewidth}
\vspace{-11pt}
\begin{tabular}{ll}
 ($\boldsymbol{E_+}$) & $\qp < \pp$ \\
 &  \\[-5pt]
 ($\boldsymbol{E_-}$) & $\ppm < \qm$  \\
 &  \\[-5pt]
 ($\boldsymbol{E_\bal}$) & $\ppm+\qp < \pp+\qm$ 
\end{tabular}
\end{minipage}
\hspace{0.02\linewidth}
\begin{minipage}{0.6\linewidth}
\begin{tabular}{ll}
&  \\[-8pt]
 ($\boldsymbol{E_\vol}$) &$\ppm+(k-1)\qm < \pp+(k-1)\qp$ \\
 &  \\[-4pt]
 $(\boldsymbol{E_\conf})$ & $\left( \frac{k\qp}{\pp+(k-1)\qp} \right) \left( \frac{k\ppm}{\ppm+(k-1)\qm} \right) < 1$ \\
 & \\[-5pt]
 $(\boldsymbol{E_G})$ & $\left( \frac{k\qp}{\pp+(k-1)\qp} \right) \left(1 + \frac{\ppm - \qm}{\ppm + (k-1)\qm} \right) < 1$\\
  & \\[-8pt]
\end{tabular}
\end{minipage}
\rule{\linewidth}{1pt}
\end{minipage}
\caption{
\vspace*{10mm}
\pedro{Conditions for the Stochastic Block Model analysis of Section~\ref{sec:stochastic_block_model}}\\[-40pt]
}
\label{tab:relations}
\end{table}
Consider the relations in Table \ref{tab:relations}.
\old{
Conditions $E_+$ and $E_-$ describe the presence of assortative or disassortative clusters in expectation. 
Note that, by Definition~\ref{definition:kBalancedSignedGraph}, a graph is balanced if and only if $\qp=\ppm=0$. 
%
\pedronew{We can see that if $E_+\cap E_-$ then $G^-$ and $G^+$ give information about the cluster structure.}
Similarly $E_\conf$ characterizes a graph where the relative amount of conflicts - \textit{i.e.}\ positive edges between 
the clusters and negative edges inside the clusters - is small. 
Condition $E_G$ is strictly related to such setting. 
In fact when $E_G$ holds and the negative graph satisfies $E_-$, then $E_\conf$ holds as well. On the other hand, 
condition $E_\bal$ is related to the balanced setting $E_{+}\cap E_-$, however it is just a necessary condition for 
the graph to be $k$-balanced in expectation, that is $E_+\cap E_-$ implies $E_\bal$, whereas the opposite is not 
necessarily true. Finally condition $E_\vol$ implies that the expected volume in the negative graph is smaller 
than the expected volume in the positive one. This condition  is therefore not related to any signed clustering structure.  
}
\pedro{Conditions $E_+$ and $E_-$ describe the presence of assortative or disassortative clusters in expectation. 
Note that, by Definition~\ref{definition:kBalancedSignedGraph}, a graph is balanced if and only if $\qp=\ppm=0$. 
\pedronew{We can see that if $E_+\cap E_-$ then $G^-$ and $G^+$ give information about the cluster structure.}
Further, if $E_+\cap E_-$ holds then $E_\bal$ holds. 
Similarly $E_\conf$ characterizes a graph where the relative amount of conflicts - \textit{i.e.}\ positive edges between 
the clusters and negative edges inside the clusters - is small. Condition $E_G$ is strictly related to such setting. 
In fact when $E_-\cap E_G$ holds then $E_\conf$ holds.
Finally condition $E_\vol$ implies that the expected volume in the negative graph is smaller 
than the expected volume in the positive one. This condition  is therefore not related to any signed clustering structure.}

Let 
$$\boldsymbol \chi_1  = \one, \qquad \boldsymbol \chi_i = (k-1)\one_{\mathcal{C}_i}-\one_{\overline{\mathcal{ C}_i}} \, .$$
The use of $k$-means on $\boldsymbol \chi_i$, $i=1,\dots,k$  identifies the ground truth communities $\mathcal C_i$. 
As spectral clustering relies on the eigenvectors corresponding to the $k$  smallest eigenvalues (see Algorithm~\ref{alg:geometricMeanClusteringSignedNetworks})
we derive here necessary and sufficient conditions such that  in expectation the eigenvectors $ \boldsymbol \chi_i, i=1,\ldots,k$ correspond
to the $k$ smallest eigenvalues of the normalized  Laplacians introduced so far. 
In particular, we observe that condition $E_G$ affects the ordering of the eigenvalues of 
the normalized geometric mean Laplacian. Instead, the ordering of the eigenvalues of the  
operators based on the arithmetic mean is related to $E_\bal$ and $E_\vol$. 
The latter is not related to any clustering, thus 
introduces a bias in the eigenvalues ordering which reverberates into a noisy embedding 
of the data points and in turn into a significantly higher clustering error. 

\begin{thm}\label{thm:signedLaplacians}
 Let $\mathcal L_{BN}$ and $\mathcal L_{SN}$ be  the normalized Laplacians  defined in \eqref{eq:signed_laplacians} of the expected graphs. The following statements are equivalent:
 \begin{enumerate}[topsep=-3pt]\setlength\itemsep{-3pt}
   \item $\boldsymbol \chi_1, \ldots, \boldsymbol \chi_k$ are the eigenvectors corresponding to the $k$ smallest eigenvalues of $\mathcal{L}_{BN}$.
   \item $\boldsymbol \chi_1, \ldots, \boldsymbol \chi_k$ are the eigenvectors corresponding to the $k$ smallest eigenvalues of $\mathcal{L}_{SN}$.
   \item The two conditions $E_\bal$ and $E_\vol$  hold simultaneously.
 \end{enumerate}
 \end{thm}
\myComment{
\begin{proof}
We first prove that $\boldsymbol \chi_1, \ldots, \boldsymbol \chi_k$ are the eigenvectors corresponding to the $k$ smallest eigenvalues 
of $\mathcal{L}_{BN}$ if and only if the two conditions $E_\bal$ and \old{$E_vol$} \pedro{$E_\vol$} hold simultaneously. 
It is simple to verify that $\boldsymbol \chi_i$ are eigenvectors of $\mathcal{W}^+$ and $\mathcal{W}^-$, 
with eigenvalues denoted by $\lambda^+_i$ and  $\lambda^-_i$, respectively. 
Thus $\mathcal{W}^+$ and $\mathcal{W}^-$ are simultaneously diagonalizable, that is there exists a 
non-singular matrix $\Sigma$ such that $\Sigma^{-1}\mathcal{W}^{\pm}\Sigma =\Lambda^\pm$, 
where $\Lambda^+$ and $\Lambda^-$ are diagonal matrices  $\Lambda^\pm = \diag(\lambda_1^{\pm}, \dots, \lambda_k^\pm, 0, \dots, 0)$.
Observe that the eigenvalues $\lambda_i^+$ and $\lambda_i^-$ admits the following explicit representations
\begin{equation}\label{eq:adjacencyMatrix_eigenvalues_eigenvectors}
\begin{aligned}
   & \lambda^+_1 = \abs{\mathcal{C}}(\pp+(k-1)\qp), \qquad \lambda^-_1 = \abs{\mathcal{C}}(\ppm+(k-1)\qm)\\
   & \lambda^+_i = \abs{\mathcal{C}}(\pp-\qp) \qquad\qquad\quad\,\,\, \lambda^-_i = \abs{\mathcal{C}}(\ppm-\qm),
\end{aligned}
\end{equation}
for $i=2, \dots, k$. As we assume clusters of the same size, the nodes have the same degree in expectation, inducing a regular graph. Hence the expected degrees 
of the graph are $d^+ = \mathcal W^+\one \pedro{= \abs{\mathcal{C}}(\pp+(k-1)\qp)\one}$, 
$d^-=\mathcal W^-\one = \pedro{\abs{\mathcal{C}}(\ppm+(k-1)\qm)\one}$ 
and 
$\bar d = d^+ + d^-$. 
With corresponding degree matrices $\mathcal{D}^+ = d^+ I$ and $\bar{\mathcal{D}}= \bar d I$. The expected balanced-ratio cut Laplacian operator 
is thus given by $\mathcal{L}_{BR} = \Sigma( d^+I - \Lambda^+ + \Lambda^- )\Sigma^{-1}$. It follows that the eigenvalues of $\mathcal{L}_{BR}$ correspond to eigenvectors in the following way
\begin{equation*}
\begin{cases}
 d^+ - \lambda_{i}^{+} + \lambda_{i}^{-}& \text{with eigenvector }\boldsymbol\chi_i, i=1,\dots, k\\
 d^+ & \text{corresponding to the remaining eigenvectors}
\end{cases}
\end{equation*}
Thus, eigenvectors $\boldsymbol\chi_i, i=1,\dots, k$ correspond to the smallest eigenvalues if and only if
\begin{equation*}
 d^+ - \lambda_{i}^{+} + \lambda_{i}^{-} < d^+ \iff \lambda_{i}^{-} < \lambda_{i}^{+}
\end{equation*}

By Eqs.~\eqref{eq:adjacencyMatrix_eigenvalues_eigenvectors} we see that for the constant eigenvector we have
\begin{equation*}
 \begin{aligned}
   \lambda_{1}^{-} < \lambda_{1}^{+} \iff d^{-} < d^{+}\, \pedro{\iff  \ppm+(k-1)\qm < \pp+(k-1)\qp\,}, 
 \end{aligned}
\end{equation*}
whereas for the  eigenvectors $\boldsymbol\chi_i, i=2,\dots, k$ the corresponding condition is
\begin{equation*}
 \begin{aligned}
   \lambda_{i}^{-} < \lambda_{i}^{+} \iff \ppm+\qp < \pp+\qm\, .
 \end{aligned}
\end{equation*}
We deduce that the  eigenvectors $\boldsymbol\chi_i, i=1,\dots, k$ correspond to the smallest eigenvalues of $\mathcal{L}_{BR}$ if and only if 
\old{$d^-<d^+$}\pedro{$\ppm+(k-1)\qm < \pp+(k-1)\qp$}  and $\ppm+\qp < \pp+\qm$. 

As $\mathcal{L}_{BN}$ differs from $\mathcal{L}_{BR}$ by a constant factor, the conditions hold for $\mathcal{L}_{BN}$.
Conditions for $\mathcal{L}_{SN}$ can be proved in the same way, as the only difference in the eigenvalues is a shift given by the degree vector $\bar{d}$.
 \end{proof}
}
\begin{thm}\label{thm:geometricMeanLaplacian}
Let $\mathcal L_{GM}=\mathcal{L}_{\sym}^+ \# \mathcal{Q}_\sym^-$ be the geometric mean of the Laplacians of the expected graphs. Then $\boldsymbol \chi_1, \ldots, \boldsymbol \chi_k$ are the eigenvectors corresponding to the $k$ smallest eigenvalues of $\mathcal L_{GM}$ if and only if condition $E_G$ holds. 
\end{thm}
\myComment{
\begin{proof}
 We use the same notation as in the \pedronew{proof of Theorem~\ref{thm:signedLaplacians}}. Observing that $\mathcal{L}_{\sym}^+$ and $\mathcal{Q}_\sym^-$ have the same eigenvectors, it follows from Theorem  \ref{thm:AB-mEigenvectors} that
\begin{equation}\label{eq:geometricMean_in_expectation}
 \mathcal L_{GM}= \Sigma\, \sqrt{( I - \widehat{\Lambda}^+ )(I + \widehat{\Lambda}^- )}\, \Sigma^{-1}
\end{equation}
where ${d^+}\widehat{\Lambda}^+ =  \Lambda^+$, and ${d^-}\widehat{\Lambda}^- = \Lambda^-$. We deduce  that the eigenvalues of $\mathcal L_{GM}$ correspond to eigenvectors in the following way
 \begin{equation*}
\begin{cases}
  \sqrt{\Big(1 - \frac{\lambda_{i}^{+}}{d^{+}} \Big) \Big(1+\frac{\lambda_{i}^{-}}{d^{-}} \Big)}   & \text{with eigenvector }\boldsymbol\chi_i, i=1,\dots, k\\
 1 & \text{corresponding to the remaining eigenvectors}
\end{cases}
\end{equation*}
Thus, eigenvectors $\boldsymbol\chi_i, i=1,\dots, k$ correspond to the smallest eigenvalues if and only if 
\begin{equation*}
 \Big(1 - \frac{\lambda_{i}^{+}}{d^{+}} \Big) \Big(1+\frac{\lambda_{i}^{-}}{d^{-}} \Big) < 1 
\end{equation*}
By eqs.~\eqref{eq:adjacencyMatrix_eigenvalues_eigenvectors} we see that for the constant eigenvector $\boldsymbol\chi_1$ we have
\begin{equation*}
 \begin{aligned}
   \Big(1 - \frac{\lambda_{1}^{+}}{d^{+}} \Big) \Big(1+\frac{\lambda_{1}^{-}}{d^{-}} \Big) = \Big(1 - \frac{d^{+}}{d^{+}} \Big) \Big(1+\frac{d^{-}}{d^{-}} \Big) = 0 < 1\, .
 \end{aligned}
\end{equation*}
For eigenvectors $\boldsymbol\chi_i, i=2,\dots, k$ first observe that
\begin{equation*}
\begin{aligned}
  1 - \frac{\lambda_{i}^{+}}{d^{+}} = (d^{+} - \lambda_{i}^{+})/d^{+} &= \Big(d^{+} - \abs{\mathcal{C}}(\pp-\qp)\Big)/d^{+} = \frac{k \qp}{\pp+(k-1)\qp} 
\end{aligned}
\end{equation*}
In the same way we have 
\begin{equation*}
\begin{aligned}
  1 + \frac{\lambda_{i}^{-}}{d^{-}} = 1 + \frac{\ppm - \qm}{\ppm + (k-1)\qm}\, . 
\end{aligned}
\end{equation*}


Thus, for the eigenvectors $\boldsymbol\chi_i, i=2,\dots, k$ we have the following condition
\begin{equation*}
 \begin{aligned}
   \Big(1 - \frac{\lambda_{i}^{+}}{d^{+}} \Big) \Big(1+\frac{\lambda_{i}^{-}}{d^{-}} \Big) < 1 \iff \Big( \frac{k\qp}{\pp+(k-1)\qp} \Big) \Big(1 + \frac{\ppm - \qm}{\ppm + (k-1)\qm} \Big) < 1\, ,
 \end{aligned}
\end{equation*}
which implies in turn that the  eigenvectors $\boldsymbol\chi_i, i=1,\dots, k$ correspond to the smallest eigenvalues of $\mathcal L_{GM}$ if and only if $E_G$ holds.
\end{proof}
}
\nosupplementary{Conditions for the geometric mean Laplacian of diagonally shifted Laplacians are available in the supplementary material.}
\myComment{
\pedronew{
\fra{As mentioned above, in practical implementations one modifies the Laplacians defining $L_{GM}$ by adding a small diagonal shift. This is done to ensure positive definiteness of the matrices. } The next theorem shows \fra{how to extend} the previous result \fra{to} the case of diagonall\fra{y} shifted Laplacians.
\newpage
\begin{thm}\label{thm:geometricMeanLaplacianDSHIFTED}
Let $\mathcal L_{GM}=(\mathcal{L}_{\sym}^+ +\varepsilon_1 I)\# (\mathcal{Q}_\sym^- +\varepsilon_2 I)$ be the geometric mean of the shifted Laplacians of the expected graphs. Then $\boldsymbol \chi_1, \ldots, \boldsymbol \chi_k$ are the eigenvectors corresponding to the $k$ smallest eigenvalues of $\mathcal L_{GM}$ if the following conditions hold.
 \begin{enumerate}[topsep=-3pt]\setlength\itemsep{-3pt}
   \item $\varepsilon_1 + \varepsilon_2 < 1$.
   \item $\left( \frac{k\qp}{\pp+(k-1)\qp} \right) \left(1 + \frac{\ppm - \qm}{\ppm + (k-1)\qm} \right) + \left( \varepsilon_1 + \varepsilon_2 \right)< 1$.
 \end{enumerate}
\end{thm}
\begin{proof}
 We use the same notation as in the previous proof. Observing that $\mathcal{L}_{\sym}^+$ and $\mathcal{Q}_\sym^-$ have the same eigenvectors, it follows from Theorem  \ref{thm:AB-mEigenvectors} that
\begin{equation}\label{eq:geometricMean_in_expectationSHIFTED}
 \mathcal L_{GM}= \Sigma\, \sqrt{( I - \widehat{\Lambda}^+ +\varepsilon_1 I)(I + \widehat{\Lambda}^- +\varepsilon_2 I)}\, \Sigma^{-1}
\end{equation}
where ${d^+}\widehat{\Lambda}^+ =  \Lambda^+$, and ${d^-}\widehat{\Lambda}^- = \Lambda^-$. We deduce  that the eigenvalues of $\mathcal L_{GM}$ correspond to eigenvectors in the following way
 \begin{equation*}
\begin{cases}
  \sqrt{\Big(1 - \frac{\lambda_{i}^{+}}{d^{+}} +\varepsilon_1 \Big) \Big(1+\frac{\lambda_{i}^{-}}{d^{-}}+\varepsilon_2  \Big)}   & \text{with eigenvector }\boldsymbol\chi_i, i=1,\dots, k\\
 (1+\varepsilon_1)(1+\varepsilon_2) & \text{corresponding to the remaining eigenvectors}
\end{cases}
\end{equation*}
Thus, eigenvectors $\boldsymbol\chi_i, i=1,\dots, k$ correspond to the smallest eigenvalues if and only if 
\begin{equation}\label{eq:eigenvalueConditionsSHIFTED}
 \Big(1 - \frac{\lambda_{i}^{+}}{d^{+}} +\varepsilon_1 \Big) \Big(1+\frac{\lambda_{i}^{-}}{d^{-}} +\varepsilon_2 \Big) < (1+\varepsilon_1)(1+\varepsilon_2)
\end{equation}
Further, we can see that the previous equation holds if and only if
\begin{equation*}
  \begin{aligned}
 \Big(1 - \frac{\lambda_{i}^{+}}{d^{+}} \Big)  \Big(1+\frac{\lambda_{i}^{-}}{d^{-}} \Big) + \varepsilon_1\Big(1+\frac{\lambda_{i}^{-}}{d^{-}} \Big)+\varepsilon_2\Big(1 - \frac{\lambda_{i}^{+}}{d^{+}} \Big) < 1 + \varepsilon_1 + \varepsilon_2
 \end{aligned}
\end{equation*}
More over, as $\Big(1 - \frac{\lambda_{i}^{+}}{d^{+}} \Big), \Big(1+\frac{\lambda_{i}^{-}}{d^{-}} \Big) \in[0,2]$, we can see that eq.\eqref{eq:eigenvalueConditionsSHIFTED} holds if
\begin{equation*}
  \begin{aligned}
 \Big(1 - \frac{\lambda_{i}^{+}}{d^{+}} \Big)  \Big(1+\frac{\lambda_{i}^{-}}{d^{-}} \Big) + \varepsilon_1 + \varepsilon_2 < 1
 \end{aligned}
\end{equation*}
\\
By eqs.~\eqref{eq:adjacencyMatrix_eigenvalues_eigenvectors} we see that for the constant eigenvector $\boldsymbol\chi_1$ we have
$1 - \frac{\lambda_{1}^{+}}{d^{+}} = 1 - \frac{d^{+}}{d^{+}} = 0$.
%
Thus,
\begin{equation*}
   \Big(1 - \frac{\lambda_{i}^{+}}{d^{+}} \Big)  \Big(1+\frac{\lambda_{i}^{-}}{d^{-}} \Big) + \varepsilon_1 + \varepsilon_2 = \varepsilon_1 + \varepsilon_2  < 1 
\end{equation*}
For eigenvectors $\boldsymbol\chi_i, i=2,\dots, k$ first observe that
\begin{equation*}
\begin{aligned}
  1 - \frac{\lambda_{i}^{+}}{d^{+}} = (d^{+} - \lambda_{i}^{+})/d^{+} &= \Big(d^{+} - \abs{\mathcal{C}}(\pp-\qp)\Big)/d^{+} = \frac{k \qp}{\pp+(k-1)\qp} 
\end{aligned}
\end{equation*}
In the same way we have 
\begin{equation*}
\begin{aligned}
  1 + \frac{\lambda_{i}^{-}}{d^{-}} = 1 + \frac{\ppm - \qm}{\ppm + (k-1)\qm}\, . 
\end{aligned}
\end{equation*}
%
%
Thus, for the eigenvectors $\boldsymbol\chi_i, i=2,\dots, k$ we have the following condition
\begin{equation*}
 \begin{aligned}
   \Big(1 - \frac{\lambda_{i}^{+}}{d^{+}} \Big) & \Big(1+\frac{\lambda_{i}^{-}}{d^{-}} \Big) +\varepsilon_1 +\varepsilon_2< 1 \iff \\ 
   \Big( \frac{k\qp}{\pp+(k-1)\qp} \Big) & \Big(1 + \frac{\ppm - \qm}{\ppm + (k-1)\qm} \Big)+\varepsilon_1+\varepsilon_2 &< 1\, ,
 \end{aligned}
\end{equation*}
This implies in turn that the eigenvectors $\boldsymbol\chi_i, i=1,\dots, k$ correspond to the smallest eigenvalues of $\mathcal L_{GM}$ if the following conditions hold
 \begin{enumerate}[topsep=-3pt]\setlength\itemsep{-3pt}
   \item $\varepsilon_1 + \varepsilon_2 < 1$.
   \item $\left( \frac{k\qp}{\pp+(k-1)\qp} \right) \left(1 + \frac{\ppm - \qm}{\ppm + (k-1)\qm} \right) + \left( \varepsilon_1 + \varepsilon_2 \right)< 1$.
 \end{enumerate}
\end{proof}
}
}
Intuition suggests that a good model should easily identify clusters when $E_+\cap E_-$. However, unlike condition $E_G$, condition $E_\vol\cap E_{\bal}$ is 
not directly satisfied under that regime. Specifically, we have 
\begin{corollary}\label{corollary:proportion_of_cases}
  Assume that $E_+\cap E_-$ holds.  
  Then $\boldsymbol \chi_1, \ldots, \boldsymbol \chi_k$ are eigenvectors  corresponding to the $k$ smallest eigenvalues of $\mathcal{L}_{GM}$. 
  \pedronew{Let} $p(k)$ denote the proportion of cases where $\boldsymbol \chi_1, \ldots, \boldsymbol \chi_k$ are the eigenvectors of the $k$ smallest eigenvalues of $\mathcal{L}_{SN}$ or $\mathcal L_{BN}$, then 
  \old{$p(k)$} $\pedro{p}(k)\leq \pedronew{\frac 1 6 + \frac{2}{3(k-1)}+\frac{1}{(k-1)^2}}$. 
\end{corollary}
\myComment{
\pedronew{
\begin{proof}
The event $E_\vol$ is defined as
\[ E_\vol = \{ (\ppm,\qm,\pp,\qp) \in [0,1]^4 \,|\, \ppm + (k-1) \qm < \pp + (k-1) \qp\}\]
We can rewrite the inequality as
\[ \qm - \qp < \frac{1}{k-1} \Big( \pp - \ppm \Big) < \frac{1}{k-1}.\]
Thus the event $\tilde E_\vol$ defined as
\[ \tilde E_\vol = \{ (\ppm,\qm,\pp,\qp) \in [0,1]^4 \,|\, \qm - \qp <   \frac{1}{k-1} \} ,\]
satisfies $E_\vol \subset \tilde E_\vol$. Then with
\begin{align*}
   E_3 &= E_+ \cap E_- = \{ (\ppm,\qm) \in [0,1]^2 \,|\, \ppm < \qm\} \cap \{ (\pp,\qp) \in [0,1]^2 \,|\,\qp < \pp\}\\
   E_B &= \{(\ppm,\qm,\pp,\qp) \in [0,1]^4 \,|\, \ppm + \qp < \pp +  \qm\}
\end{align*}
we observe $E_3 \subset E_B$. Then
\begin{align*}
 p(k) &= \Pr( E_B \cap E_\vol \,|\, E_3) = \frac{\Pr( E_B \cap E_\vol \cap E_3)}{\Pr(E_3)} = \frac{\Pr( E_\vol \cap E_3)}{\Pr(E_3)}\\
      &\leq \frac{\Pr( \tilde E_\vol \cap E_3)}{\Pr(E_3)}
\end{align*}
Then we get with $(x_1,x_2,x_3,x_4)$ corresponding to $(\pp,\qp,\qm,\ppm)$
\begin{align*}
\Pr( \tilde E_\vol \cap E_3) &\leq \int_0^1 \Big( \int_0^{x_1} \Big( \int_0^{x_2+\frac{1}{k-1}} \Big(\int_0^{x_3} dx_4\Big) dx_3 \Big) dx_2 \Big) dx_1\\
                             &= \int_0^1 \Big( \int_0^{x_1} \Big( \int_0^{x_2+\frac{1}{k-1}} x_3 dx_3 \Big) dx_2 \Big) dx_1\\
                             &= \int_0^1 \Big( \int_0^{x_1} \Big( \frac{1}{2}\Big(x_2+\frac{1}{k-1}\Big)^2 \Big) dx_2 \Big) dx_1\\
                             &= \int_0^1 \Big[\frac{1}{6}\Big(x_1+\frac{1}{k-1}\Big)^3\Big]_0^{x_1} dx_1\\
                             &= \int_0^1 \Big[\frac{x_1^3}{6} + \frac{x_1^2}{2(k-1)} + \frac{x_1}{2(k-1)^2}\Big] dx_1\\
                             &= \Big[ \frac{x_1^4}{24} + \frac{x_1^3}{6(k-1)} + \frac{x_1^2}{4(k-1)^2}\Big]_0^1\\
                             &= \frac{1}{24} + \frac{1}{6(k-1)} + \frac{1}{4(k-1)^2}
\end{align*}
The first inequality comes from the fact that we do not ensure that the integration upper border for $x_3$ is smaller or equal to one. 
Thus with $\Pr(E_3)=\frac{1}{4}$ we get
\[ p(k) \leq \frac{1}{6} + \frac{2}{3(k-1)} + \frac{1}{(k-1)^2}.\]
\end{proof}
}
\old{
\begin{proof}
We easily see that $E_G$ holds when $\ppm<\qm$ and $\qp<\pp$.  Consider the following events
  \begin{gather*}
   E_3=E_+\cap E_-=\{\ppm < \qm\}\cap\{\qp<\pp\}\quad \text{and} \quad \tilde E_\vol=\{ \qm -\qp < 1/(k+1) \}
  \end{gather*}
  Note that $E_3 \subseteq E_\bal$ and $E_\vol\subseteq \tilde E_\vol$. The probability $p(k)$ that $E_\bal$ and $E_\vol$ hold  simultaneously, conditioned to $E_3$  is given by the probability of the event $E_\bal \cap E_\vol$ divided by the probability of $E_3$. We have the following upper bound
  $$p(k) =\mathbb P(E_\bal\cap E_\vol|E_3)=\frac{\mathbb{P}(E_\bal\cap E_\vol)}{\mathbb{P}(E_3)}\leq \frac{\mathbb{P}(E_\bal\cap \tilde E_\vol)}{\mathbb{P}(E_3)}\, .$$
  Therefore
  \begin{align*}
   \mathbb{P}(E_\bal\cap \tilde E_\vol) =\int_{E_\bal\cap \tilde E_\vol} \!\!\!\!\!\!d\x = \int_0^1 \!\!dx_1 \int_0^{x_1} \!\!dx_2 \int_0^{x_2+\frac 1 {k+1}}\!\!dx_3\int_0^{x_3} \!\!dx_4 = \frac 1 {24}+\frac{1}{6(k+1)}+\frac{1}{2(k+1)^2}
  \end{align*}
  and 
  $$\mathbb{P}(E_3)= \int_{E_3}d \x = \left(\int_0^1 dx_1 \int_0^{x_1} dx_2\right)^2=\frac 1 4 $$
  concluding the proof.
\end{proof}
}  
}
In order to grasp the difference in expectation between $L_{BN}$, $L_{SN}$ and $L_{GM}$, 
in Fig~\ref{fig:SBM} we present the proportion of cases where 
Theorems~\ref{thm:signedLaplacians} and~\ref{thm:geometricMeanLaplacian} hold under different contexts.
Experiments are done with all four parameters discretized in $[0,1]$ with 100 steps. The expected proportion of cases where $E_G$ 
holds (Theorem \ref{thm:geometricMeanLaplacian}) is far above the corresponding proportion for $E_\vol \cap E_\bal$ (Theorem \ref{thm:signedLaplacians}), 
showing that in expectation the geometric mean Laplacian is superior to the other signed Laplacians.
In Fig.~\ref{fig:SBM-Sampled} we present experiments on \pedronew{sampled} graphs with $k$-means on top of the $k$ smallest eigenvectors.
In all cases we consider clusters of size $\abs{\mathcal C} = 100$ and present the median of clustering error \pedro{(i.e., error when clusters are labeled via majority vote)} of $50$ runs. 
The results show that the analysis made in expectation closely resembles the actual behavior. 
In fact, even if we expect only one noisy eigenvector for $L_{BN}$ and $L_{SN}$, 
the use of the geometric mean Laplacian significantly outperforms any other previously 
proposed technique in terms of clustering error. 
\old{The balance normalized 
Laplacian and the signed normalized Laplacian}
\pedro{$L_{SN}$ and $L_{BN}$}
achieve good clustering only when the 
graph resembles a $k$-balanced structure, whereas they fail  
even in the ideal situation where either the positive or the negative graphs 
\pedronew{are informative about the cluster structure.}
As shown in Section \ref{sec:real_world_experiments}, the advantages of $L_{GM}$ over the other Laplacians 
discussed so far allow us to identify  a clustering structure on the Wikipedia benchmark real world signed network, 
where  other clustering approaches have failed.  
\begin{figure}[!t]
\centering
\includegraphics[width=1\columnwidth, clip]{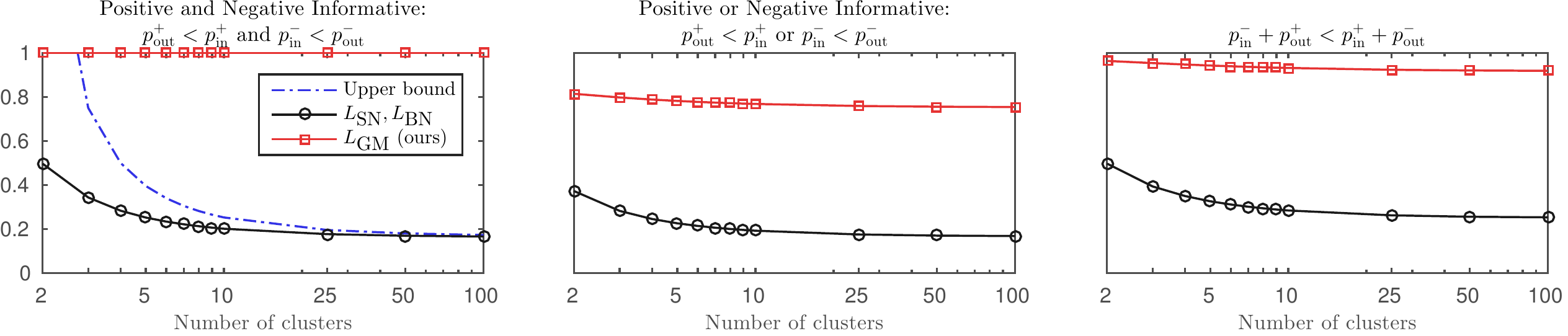}
\caption{Fraction of cases where in expectation $\boldsymbol \chi_1, \ldots, \boldsymbol \chi_k$ 
correspond to the $k$ smallest eigenvalues under the SBM. \nosupplementary{Upper bounds for middle and right plots are discussed in the supplementary material.} 
}
\label{fig:SBM}
\end{figure}
\begin{figure}[!t]
\centering
\includegraphics[width=1\columnwidth, clip]{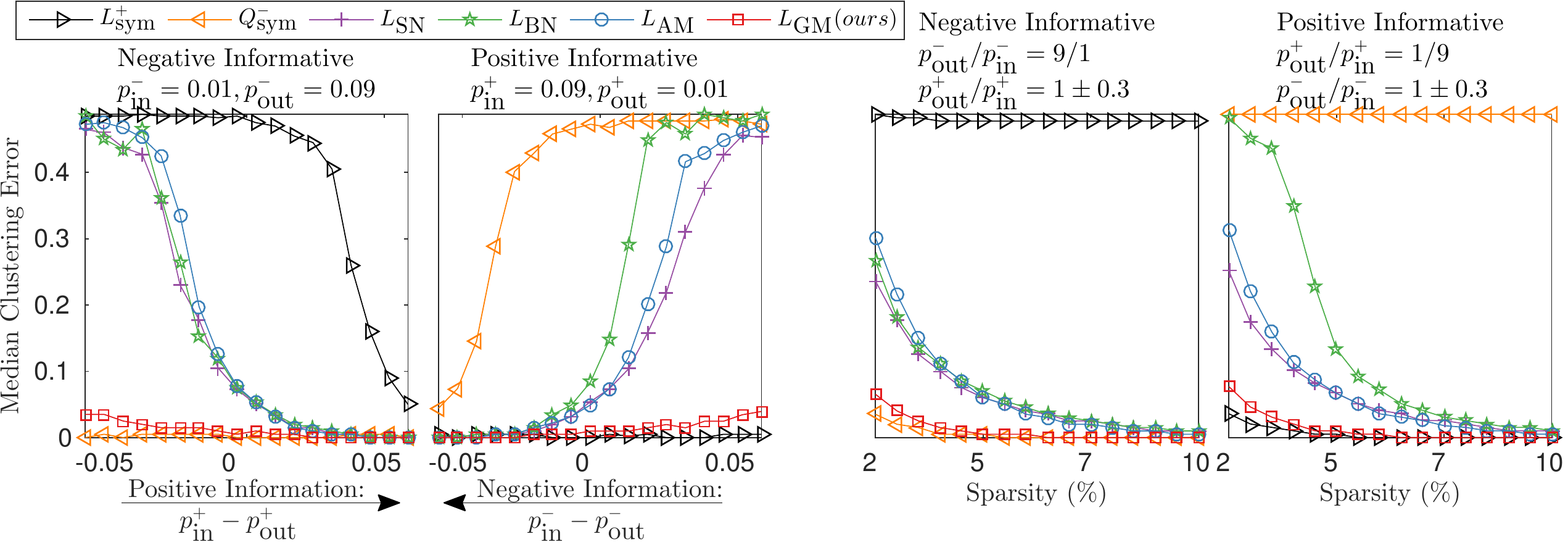}
\vspace*{-5mm}
\caption{Median clustering error under the stochastic block model over 50 runs.\\[-15pt]}
\label{fig:SBM-Sampled}
\end{figure}

\section{Krylov-based inverse power method for small eigenvalues of~$L^+_{\sym}\# Q^-_{\sym}$}\label{sec:IPM}
 
 The computation of the geometric mean $A\# B$ of two positive definite matrices of moderate size has been 
 discussed extensively by various authors \cite{raissouli:continued_fractions,higham:sign_function,Ianazzo:2012:geometricMean, Iannazzo:optimization}. 
 However, when $A$ and $B$ have large dimensions, the approaches  proposed so far become unfeasible, in fact $A\#B$ is in general a full matrix even if $A$ and $B$ are sparse.  
 In this section we present a scalable algorithm for the computation of the smallest eigenvectors of $L^+_{\sym}\# Q^-_{\sym}$. 
 The method is discussed for a general pair of matrices $A$ and $B$, to emphasize its general applicability  
 which is therefore interesting in itself. We remark that the method takes advantage of the sparsity of $A$ and $B$ 
 and does not require to explicitly compute the matrix $A\# B$. To our knowledge this is the first effective 
 method explicitly built for the computation of the eigenvectors of the geometric mean of two large and sparse positive definite matrices. 

Given a positive definite matrix $M$ with eigenvalues $\lambda_1\leq \dots\leq \lambda_n$, let $\mathcal H$ be any
eigenspace of $M$ associated to $\lambda_1, \dots, \lambda_t$.
 The inverse power method (\textbf{IPM}) applied to $M$ is a  method that 
 converges to an eigenvector $\x$ associated to the smallest eigenvalue $\lambda_{\mathcal H}$ of $M$ such that 
 $\lambda_{\mathcal H} \neq \lambda_i$, $i=1, \dots, t$.  
 The pseudocode of IPM applied to $A\# B = A(A^{-1}B)^{1/2}$ is shown in Algorithm~\ref{alg:IPM}.
 Given a vector $\v$ and a matrix $M$, the notation $\mathrm{solve}\{M,\v\}$ is used to denote a procedure returning the solution $\x$ of the linear system $M\x=\v$. 
 At each step the algorithm requires the solution of two linear systems. 
 The first one (line $2$) is solved by the preconditioned conjugate gradient method, where the preconditioner is obtained by the incomplete Cholesky decomposition of $A$.   
 Note that the conjugate gradient method is very fast, as $A$ is assumed sparse and positive definite, 
 and it is matrix-free, \textit{i.e.}\ it requires to compute the action of $A$ on a vector, whereas it does not require the knowledge of $A$ (nor its inverse). 
 The solution of the linear system occurring in line $3$ is the major inner-problem of the proposed algorithm.  
 Its efficient solution is performed by means of an extended Krylov subspace technique that we describe in the next  section. 
 The proposed implementation ensures the whole IPM is matrix-free and scalable. 

 \subsection{Extended Krylov subspace method for the solution of the linear system $(A^{-1}B)^{1/2}\x=\y$}
  We discuss here how to apply the technique known as Extended Krylov Subspace  Method ({\bf EKSM}) for 
  the solution of the linear system $(A^{-1}B)^{1/2}\x=\y$. Let $M$ be a large and sparse matrix, 
  and $\y$ a given vector.  When $f$ is a function with a single pole, 
  EKSM is a very effective method to approximate the vector $f(M)\y$\old{,} 
  without ever computing the matrix $f(M)$ \cite{druskin:1998:extendedKrylov}. 
  Note that, given two positive definite matrices $A$ and $B$ and a vector $\y$, 
  the vector we want to compute is $\x = (A^{-1}B)^{-1/2}\y$, so that our problem 
  boils down  to the  computation of the product $f(M)\y$, where $M = A^{-1}B$ and $f(X)=X^{-1/2}$.  
 The general idea of EKSM $s$-th iteration is to project $M$ onto  the subspace 
$$\mathbb K^s(M,\y) = \mathrm{span}\{\y, M\y, M^{-1}\y, \dots, M^{s-1}\y, M^{1-s}\y\}\, ,$$
and solve the problem there. The projection onto $\mathbb K^s(M,\y)$ is realized by means of the Lanczos process, 
which produces a sequence of matrices $V_s$ with orthogonal columns,  
such  that the first column of $V_s$ is a multiple of $\y$ and $\mathrm{range}(V_s)=\mathbb K^s(M,\y)$. 
Moreover at each step we have
\begin{equation}\label{eq:lanczos}
 MV_s = V_s H_s + [\u_{s+1},\v_{s+1}][\e_{2s+1}, \e_{2s+2}]^T
\end{equation}
where $H_s$ is $2s\times 2s$ symmetric tridiagonal, 
$\u_{s+1}$ and $\v_{s+1}$ are orthogonal to $V_s$, and $\e_i$ is the $i$-th canonical vector. 
The solution $\x$ is then approximated by 
$\x_s= V_s f(H_s)\e_1\|\y\|\approx f(M)\y$. 
If $n$ is the order of $M$, then the exact solution is obtained after at most $n$ steps. 
However, in practice, significantly fewer iterations are enough to achieve a good approximation, as the error $\|\x_s-\x\|$ decays exponentially with $s$ \old{\cite{simoncini:2008:error_decay}} \pedro{(Thm 3.4 and Prop. 3.6 in \cite{simoncini:2008:error_decay})}. \nosupplementary{See the supplementary material for details.}  
      
The pseudocode for the extended Krylov iteration is presented in Algorithm \ref{alg:kyrlov}. 
We use the stopping criterion proposed in \cite{simoncini:2008:error_decay}.
It is worth pointing out that  at step $4$ of the algorithm we can freely choose any scalar product $\left<\cdot ,\cdot \right>$, without affecting formula \eqref{eq:lanczos} nor the   convergence properties of the method. As $M=A^{-1}B$, we use the scalar product $\left<\u ,\v \right>_A=\u^TA\v$ induced by the positive definite matrix $A$, so that the computation of the tridiagonal matrix $H_s$  
in the algorithm simplifies to $V_s^T B V_s$. We refer to  \cite{fasi:2016:computing} for further details.  As before, the $\mathrm{solve}$ procedure  is implemented by means of the preconditioned conjugate gradient method, where the preconditioner is obtained by the incomplete Cholesky decomposition of the coefficient matrix.
Figure \ref{fig:timeComparison} shows that  we are able to compute the smallest eigenvector of $L_\sym^+\# Q_\sym^-$ being just a constant factor worse than the computation of the eigenvector of the arithmetic mean, whereas the direct computation of the geometric mean followed by the computation of the eigenvectors is unfeasible for large graphs.
\hspace*{2mm}
\RestyleAlgo{plain}
\begin{minipage}{0.95\linewidth}
 \rule{\linewidth}{1pt}
      \begin{minipage}{0.38\linewidth}
      \vspace{-42pt}
          \begin{algorithm}[H]
          \DontPrintSemicolon
		\caption{\footnotesize{IPM applied to $A\# B.^{\textcolor{white}{1/2}}$}}\label{alg:IPM}
		{\footnotesize
		\KwIn{$\x_0$,  eigenspace $\mathcal H$ of $A\# B$.}
		\KwOut{Eigenpair  $(\lambda_{\mathcal H}, \x)$ of $A\# B$}
		\Repeat{tolerance reached}{
			$\u_k$ $\gets$ $\mathrm{solve}\{A, \x_k\}$ \;
			$\v_k$ $\gets$ $\mathrm{solve}\{(A^{-1}B)^{1/2}, \u_k\}$ \;
			$\y_k$ $\gets$ project $\u_k$ over $\mathcal H^\bot$ \;
			$\x_{k+1}$ $\gets$ $\y_k / \|\y_k \|_2$ \;}
		$\lambda_{\mathcal H}$ $\gets$ $\x_{k+1}^T \x_k$,\quad  $\x$ $\gets$ $\x_{k+1}$\;
		}
		\end{algorithm}
      \end{minipage}
      \hspace{0.01\linewidth}
      \begin{minipage}{0.545\linewidth}
          \begin{algorithm}[H]
          \DontPrintSemicolon
		\caption{\footnotesize{EKSM for the computation of $(A^{-1}B)^{-1/2}\y$}}\label{alg:kyrlov}
		{\footnotesize
		\KwIn{$\u_0 = \y$, $V_0 = [\,\cdot\,  ]$ }
		\KwOut{$\x=(A^{-1}B)^{-1/2}\y$}
		$\v_0$ $\gets$ $\mathrm{solve}\{B,A\u_0\}$\;
		\For{$s=0,1,2,\dots,n$}{
			$\tilde V_{s+1}$ $\gets$ $[V_s,\u_{s}, \v_{s}]$\;
			$V_{s+1}$ $\gets$ Orthogonalize columns of $\tilde V_{s+1}$ w.r.t. $\left <\cdot ,\cdot \right>_A$\;
			$H_{s+1}$ $\gets$ $V_{s+1}^T BV_{s+1}$\;
			$\x_{s+1}$ $\gets$ $H_{s+1}^{-1/2}\e_1$\;
			\textbf{if} {\it tolerance reached} \textbf{then} {\it break} \;
			$\u_{s+1}$ $\gets$ $\mathrm{solve}\{A,BV_{s+1}\e_1\}$\;
			$\v_{s+1}$ $\gets$  $\mathrm{solve}\{B,AV_{s+1}\e_2\}$	\;	
		}
		$\x$ $\gets$ $V_{s+1}\x_{s+1}$\;
		}
		\end{algorithm}
      \end{minipage}
\rule{\linewidth}{1pt}
\end{minipage}
\begin{figure}[h]
\floatbox[{\capbeside\thisfloatsetup{capbesideposition={right,top},capbesidewidth=.62\textwidth}}]{figure}[\FBwidth]%
{\caption{Median execution time of 10 runs for different Laplacians. 
Graphs have  two perfect clusters and $2.5\%$ of edges among nodes. 
$L_{GM}(ours)$ uses Algs \ref{alg:IPM} and \ref{alg:kyrlov}, whereas we used Matlab's \texttt{eigs} for the other matrices. 
The use of \texttt{eigs} on $L_{GM}$ is prohibitive as it needs the matrix $L_{GM}$ to be built (we use the toolbox provided in~\cite{matrixMeanToolbox}), 
destroying the sparsity of the original graphs. Experiments are performed using one thread.}
    \label{fig:timeComparison}}{\includegraphics[width=.35\textwidth]{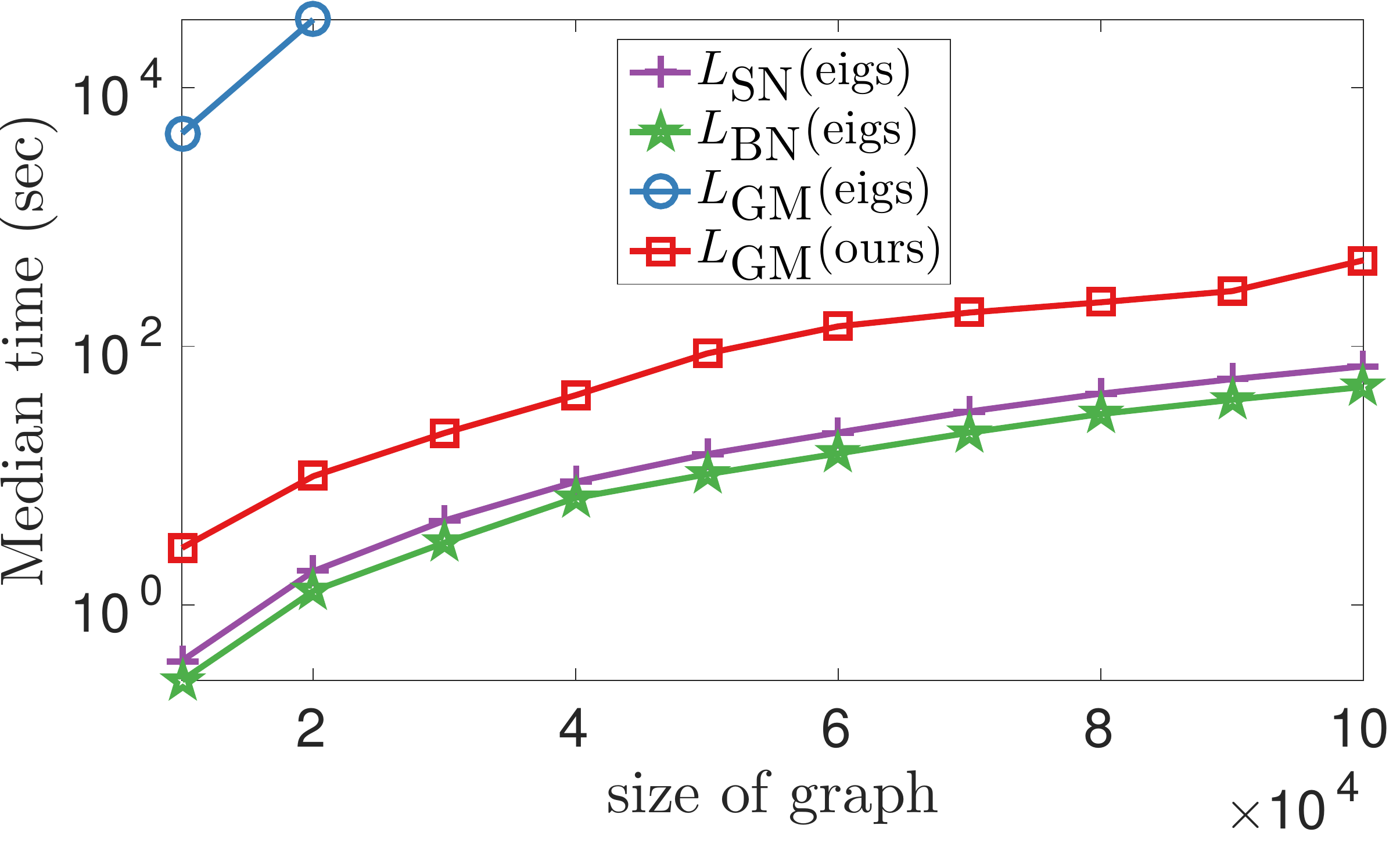}}
\end{figure}

\myComment{%
\subsection{On the computational cost of the method}
Let $c(n)$ denote the computational cost to compute the solution of a linear system with coefficient matrix either $L_{\sym}^+$ or $Q_{\sym}^-$. 
Standard iterative techniques allows to compute the smallest eigenvector of $L_{\sym}^+$ or $Q_{\sym}^-$ at a cost of $O(c(n))$ operations per step. 
\pedronew{We} show that the use of Algorithms \ref{alg:IPM} and \ref{alg:kyrlov} allows to compute the eigenvectors of $L_{\sym}^+\# Q_{\sym}^-$ with the same order of operations. 

First of all it is important to realize that the matrix $H_s = V_s^T BV_s$, defined in line $5$, 
can be defined iteratively and does not require any additional matrix multiplication \cite{simoncini:2008:error_decay}. 
Thus the cost of each iteration of  Algorithm \ref{alg:kyrlov} \pedronew{dominated} by lines $8$ and $9$, and requires $O(2c(n))$ ops. 
The algorithm converges exponentially, namely if $[a, b]$ is any interval containing the eigenvalues of $A^{-1}B$, then $\|\x_s - \x\|=O(\exp(-2s\sqrt[4]{a/b}))$, where $\x = (A^{-1}B)^{-1/2}\y$. See f.i.\ \cite{simoncini:2008:error_decay} for details. Thus $O(s_\varepsilon)$ iterations are enough to reach the prescribed tolerance  $\varepsilon >0$, where $s_\varepsilon = |\log \varepsilon/2\sqrt[4]{a/b}|$. However it is worth pointing out that in practice, at least for the matrices considered in this work, much less iterations than $O(s_\varepsilon)$ are enough. 
Therefore the proposed IPM technique allows  to compute the smallest eigenvector of $L_{\sym}^+\# Q_{\sym}^-$ at a cost of $O(c(n))+O(2 s_\varepsilon c(n))$ operations per step. 
This shows that the method is scalable. A final important remark concerns step $6$. 
The matrix $H_s$ is tridiagonal of size $2s\times 2s$, thus the function $H_s^{-1/2}$ can be implemented directly using a method for dense matrices, without any \pedronew{notable} change to the overall algorithm cost.

Next Figure \ref{fig:timeComparison} shows that, despite the computationally ugly definition of $L_\sym^+\# Q_\sym^-$,  we are able to compute its smallest eigenvector with a constant factor overcome, whereas the naive \pedronew{direct} computation would be extremely prohibitive or unfeasible.

In Fig.~\ref{fig:timeComparison} we show the median execution time for the computation of the smallest eigenvector
of the signed ratio/normalized cut Laplacians , the balance ratio/normalized cut Laplacians and the geometric mean $L^{+}_{\sym} \# Q^{-}_{\sym}$.
We randomly generate graphs with a sparsity of $2.5\%$ under the perfect stochastic case, \old{i.e.}\pedro{\textit{i.e.}} $\pp=\qm=1$ and $\qp=\ppm=0$, where 
the size of graphs goes from $10,000$ to $100,000$ in steps of $10,000$. For each setting we report the median execution time out of 10 runs. 
Experiments are performed using one thread.

For the computation of the smallest eigenvector of the signed ratio/normalized cut Laplacians and the balance ratio/normalized cut Laplacians
we compute the Laplacian matrix (\old{i.e.}\pedro{\textit{i.e.}} $L_{SR}$, $L_{SN}$, $L_{BR}$ and $L_{BN}$) and use the function \texttt{eigs} from Matlab.
For the computation of the smallest eigenvector of the geometric mean we consider two approaches:
one approach is based on the computation of the geometric mean $L_{\sym}^+\#Q_{\sym}^-$ using the Matlab toolbox provided by \cite{Ianazzo:2012:geometricMean} and then the
use of the function \texttt{eigs} from Matlab (in Fig.~\ref{fig:timeComparison} denoted as $L_{GM}$(eigs)).
The second approach is based on the Inverse Power Method of Algorithm~\ref{alg:IPM} together with the extended Krylov method 
of Algorithm~\ref{alg:kyrlov} (in Fig.~\ref{fig:timeComparison} denoted as $L_{GM}$(ours)).

We can see that the execution time of for signed Laplacians is rather similar.
One can observe that the execution time for the geometric mean with Matlab's \texttt{eigs} is truncated for graphs that have more than 20,000 nodes. This happens as the computation of the geometric mean does not fit into memory.
On the other side, the time execution for the geometric mean with the Inverse Power Method and extended Krylov methods (Algorithms~\ref{alg:IPM} and~\ref{alg:kyrlov})
  is comparable with the one of the signed Laplacians that use \texttt{eigs}. In particular it is noticeable that the  
time executions  differs just by a constant factor.
}

\section{Experiments}\label{sec:real_world_experiments}

\myComment{
\pedro{
{\bf Sociology Networks} We evaluate signed Laplacians $L_{SN}, L_{BN}, L_{AM}$ and $L_{GM}$ through three 
real-world and moderate size signed networks:
Highland tribes (Gahuku-Gama) network \cite{Read:1954:Cultures},
Slovene Parliamentary Parties Network \cite{Kropipvnik:1996:slovene} and
US Supreme Court Justices Network \cite{Doreian:2009:Partitioning}. 
For the sake of comparison we take as ground truth the clustering that is stated in the corresponding references.
We observe that all signed Laplacians yield zero clustering error.
}
}

{\bf Experiments on Wikipedia signed network.}
We consider the Wikipedia adminship election dataset from~\cite{snapnets}, 
which describes relationships that are positive, negative or non existent.
We use Algs.~\ref{alg:geometricMeanClusteringSignedNetworks}$-$\ref{alg:kyrlov} and look for 30 clusters. 
Positive and negative adjacency matrices sorted according to our clustering are depicted in Figs.~\ref{fig:wikipediaWposAll} and \ref{fig:wikipediaWnegAll}.
We can observe the presence of a large relatively empty cluster. Zooming into the denser portion of the graph 
we can see a  $k$-balanced behavior (see Figs.~\ref{fig:wikipediaWposZoom} and ~\ref{fig:wikipediaWnegZoom}), \textit{i.e.}\
the positive adjacency matrix shows assortative groups - resembling a block diagonal structure - 
 while the negative adjacency matrix shows a disassortative setting. 
Using $L_{AM}$ and $L_{BN}$ we were not able to find any clustering structure, which corroborates results reported in~\cite{Chiang:2012:Scalable}.  
This further confirms that $L_{GM}$ overcomes other clustering approaches. To the knowledge of the authors, this is the first time that
clustering structure has been found in this dataset.
\begin{figure}[!h]
\centering     
\subfigure[$W^+$]{\label{fig:wikipediaWposAll}       \includegraphics[width=0.25\linewidth,trim = 30mm 60mm 00mm 55mm, clip]{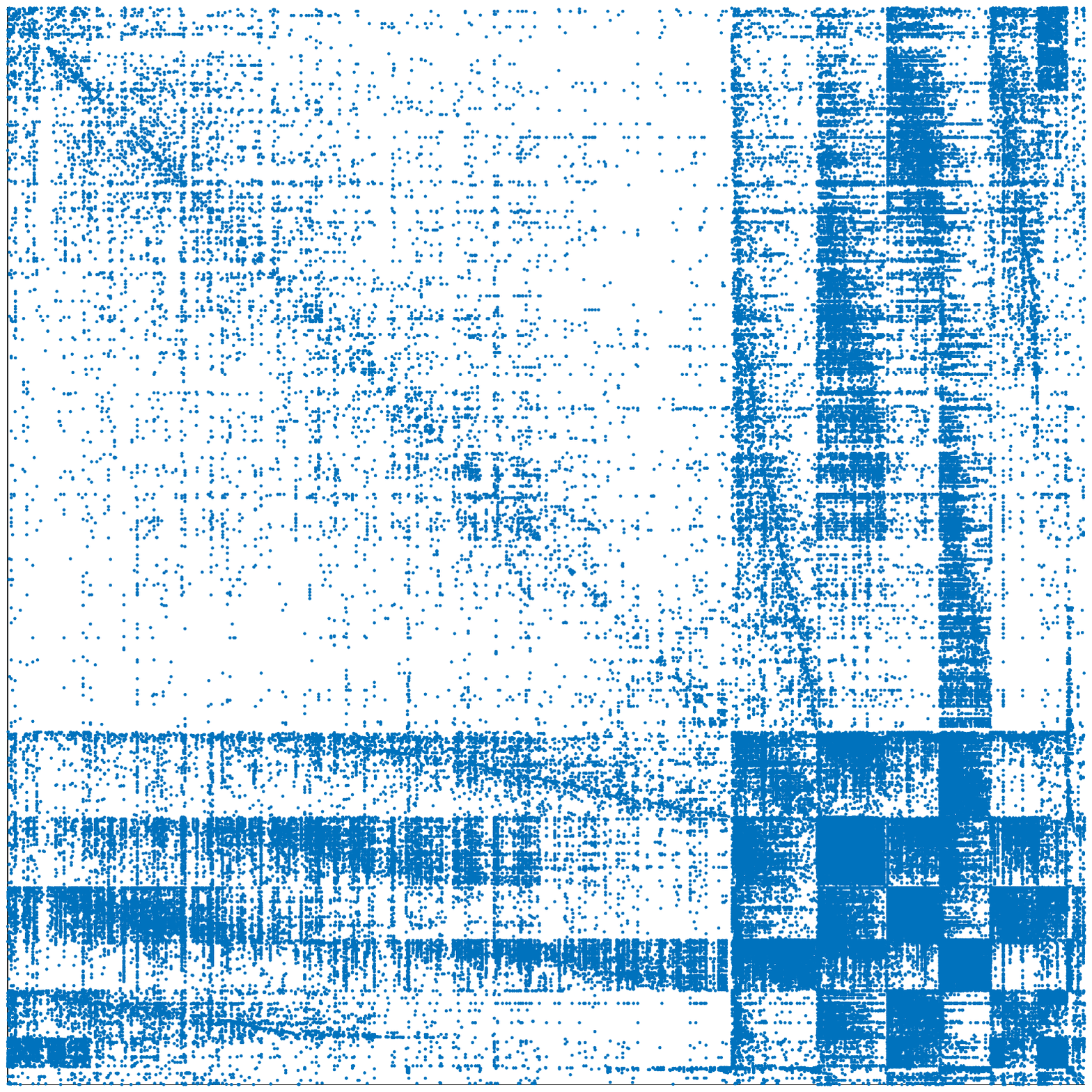}}\hspace*{\fill}
\subfigure[$W^-$]{\label{fig:wikipediaWnegAll}       \includegraphics[width=0.25\linewidth,trim = 30mm 60mm 00mm 55mm, clip]{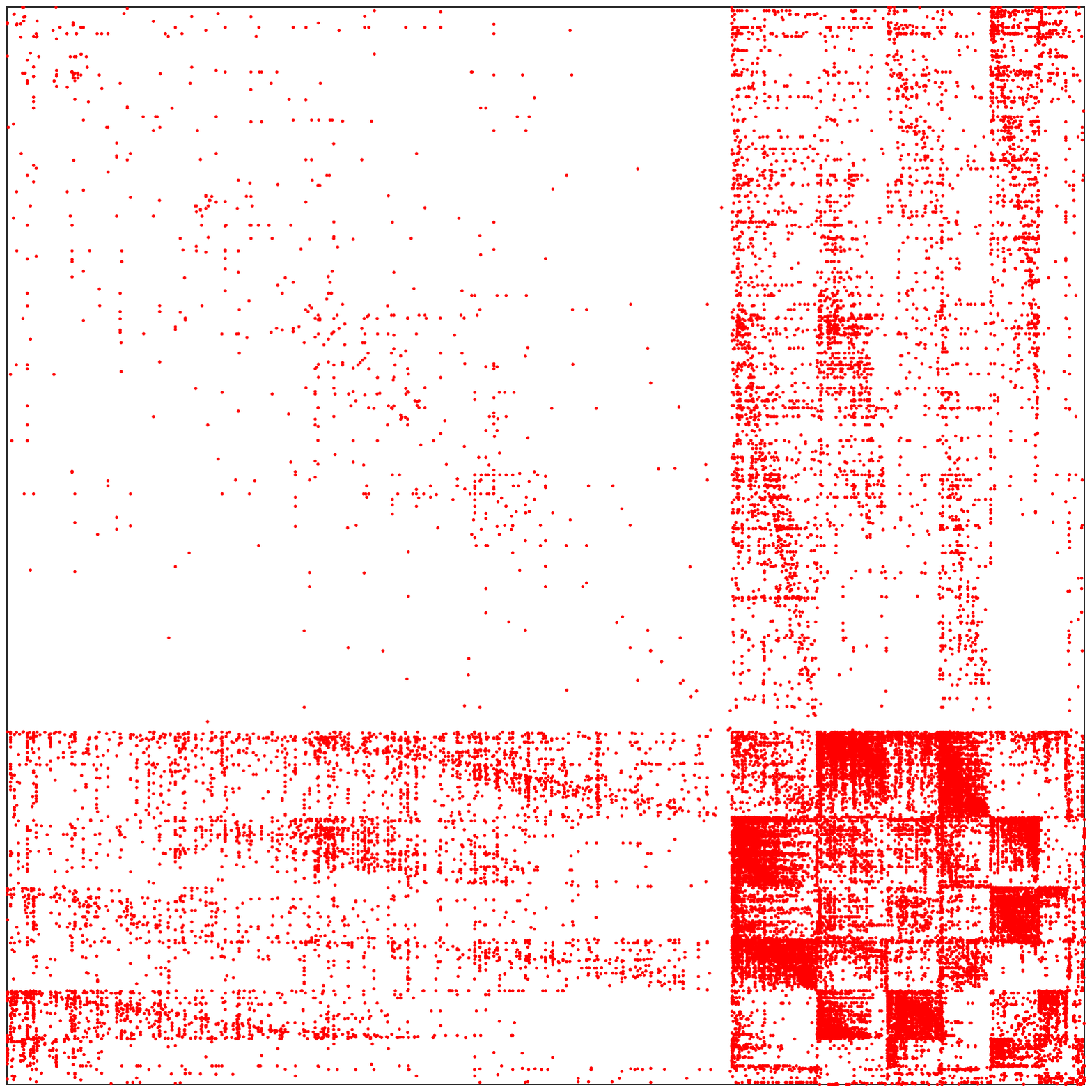}}\hspace*{\fill}
\subfigure[$W^+$(Zoom)]{\label{fig:wikipediaWposZoom}\includegraphics[width=0.25\linewidth,trim = 30mm 60mm 00mm 55mm, clip]{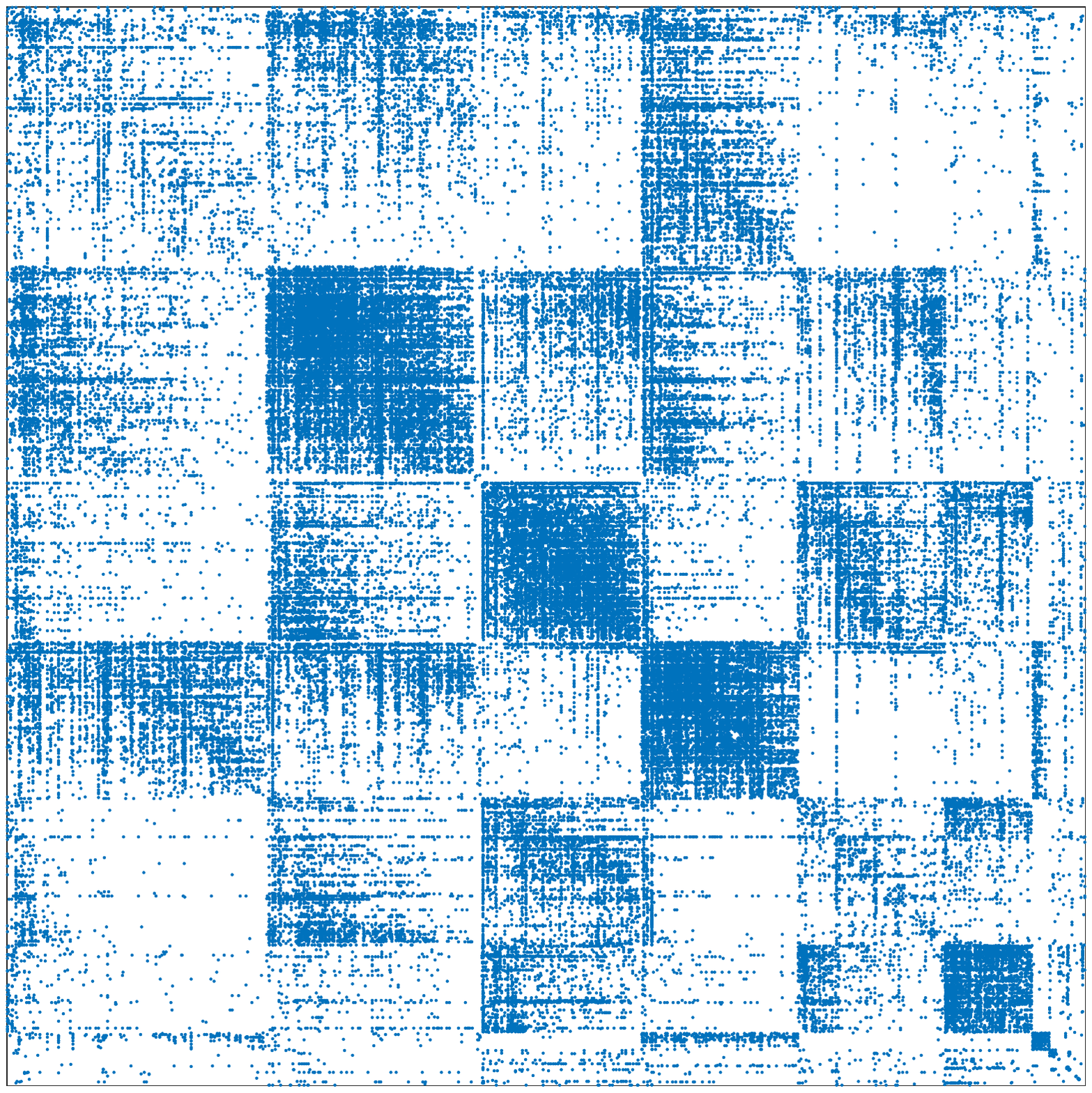}}\hspace*{\fill}
\subfigure[$W^-$(Zoom)]{\label{fig:wikipediaWnegZoom}\includegraphics[width=0.25\linewidth,trim = 30mm 60mm 00mm 55mm, clip]{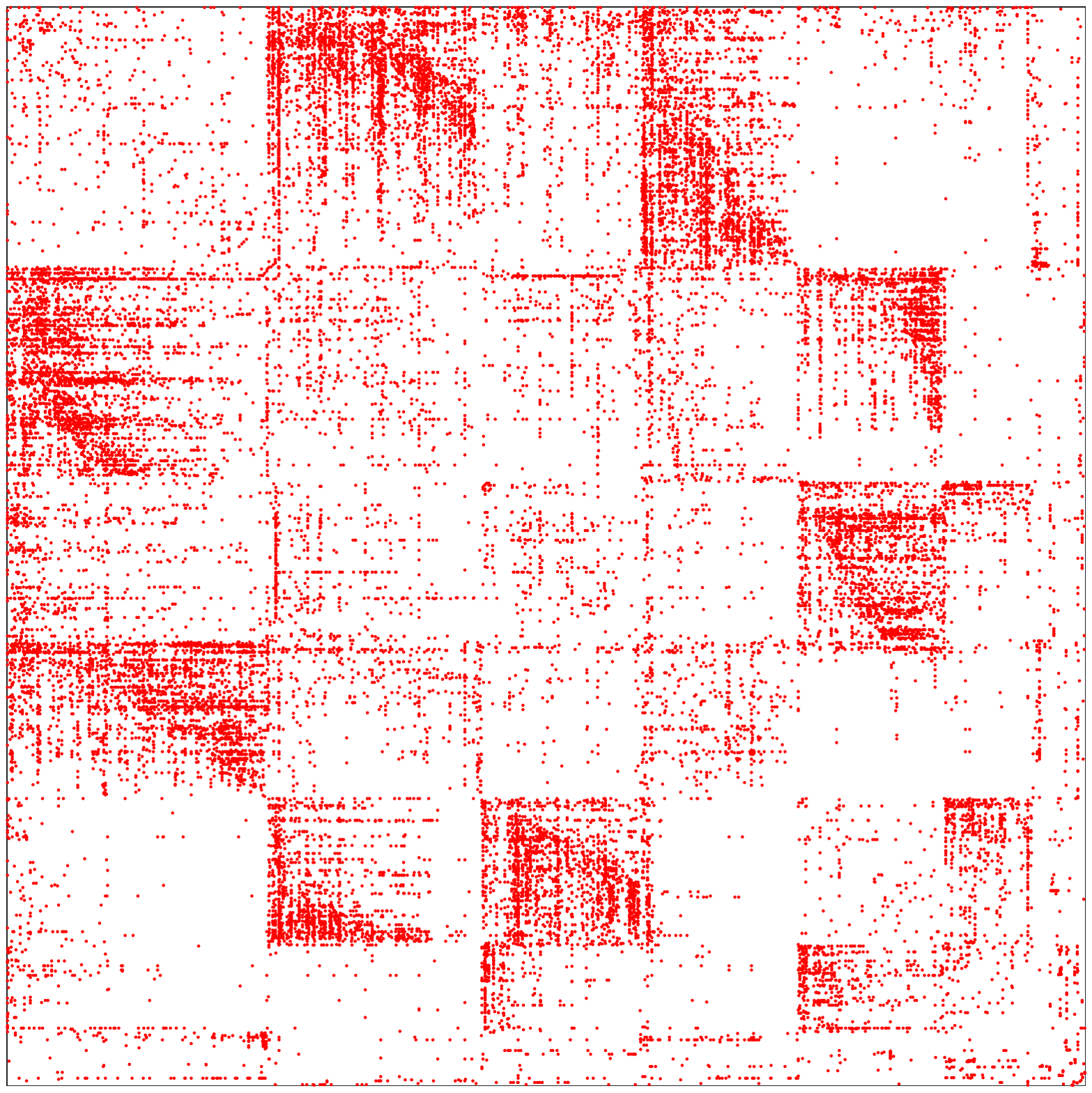}}
\caption{\vspace{-10pt}Wikipedia weight matrices sorted according to the clustering obtained with $L_{GM}$ (Alg.\ \ref{alg:geometricMeanClusteringSignedNetworks}).}
\label{fig:wikipedia}
\end{figure}

{\bf Experiments on UCI datasets.} We evaluate our method $L_{GM}$ (Algs.~\ref{alg:geometricMeanClusteringSignedNetworks}$-$\ref{alg:kyrlov}) against $L_{SN}$, $L_{BN}$, and $L_{AM}$ with  datasets from \pedronew{the} UCI repository (see Table.~\ref{table:UCIexperiments}). We build $W^+$ from a symmetric $k^+$-nearest neighbor graph, whereas $W^-$ is obtained from the symmetric $k^-$-farthest neighbor graph. For each dataset we test all clustering methods over all possible choices of $k^+, k^- \in \{3,5,7,10,15,20,40,60 \}$.
In Table~\ref{table:UCIexperiments} we report the fraction of cases where each method achieves the best and strictly best clustering error over all the 64 graphs, per each dataset. We can see that our method  outperforms other methods across all datasets.

\begin{table}[H]
\newcolumntype{"}{@{\hskip\tabcolsep\vrule width 1pt\hskip\tabcolsep}}
\begin{minipage}{\linewidth}
\begin{minipage}{.7\linewidth}
\setlength\extrarowheight{1pt}
\setlength{\tabcolsep}{3pt}
\footnotesize
\centering        
\begin{tabular}{ccccccccc}   
 \specialrule{1.5pt}{.1pt}{2pt} 
            & & \textbf{iris} & \textbf{wine} & \textbf{ecoli} & \textbf{optdig} & \textbf{USPS} & \textbf{pendig} & \textbf{MNIST} \\
\specialrule{1pt}{.1pt}{0pt} 
\footnotesize{\# vertices} & & \small{150}   & \small{178}   & \small{310}        & \small{5620}       & \small{9298}  & \small{10992}      & \small{70000}\\            
\footnotesize{\# classes}  & & \small{3}     & \small{3}     & \small{3}          & \small{10}         & \small{10}    & \small{10}         & \small{10} \\ 
\specialrule{1pt}{-1pt}{1pt} 
\multirow{2}{*}{$L_{SN}$} & \footnotesize{Best (\%)} & 23.4 & 40.6 & 18.8 & 28.1 & 10.9 & 10.9 & 12.5\\   
 & \footnotesize{Str.\ best (\%)}                 & 10.9 & 21.9 & 14.1 & 28.1 & 9.4  & 10.9 & 12.5\\   
\specialrule{.1pt}{.1pt}{.1pt}                                                                        
\multirow{2}{*}{$L_{BN}$} & \footnotesize{Best (\%)} & 17.2 & 21.9 & 7.8 & 0.0 & 1.6 & 3.1 & 0.0\\
 & \footnotesize{Str.\ best (\%)}                 & 7.8  & 4.7  & 6.3  & 0.0 & 1.6 & 3.1 & 0.0\\  
\specialrule{.1pt}{.1pt}{.1pt} 
\multirow{2}{*}{$L_{AM}$} & \footnotesize{Best (\%)} & 12.5 & 28.1 & 14.1 & 0.0 & 0.0 & 1.6 & 0.0\\   
 & \footnotesize{Str.\ best (\%)}                 & 10.9 & 14.1 & 12.5 & 0.0 & 0.0 & 1.6 & 0.0\\
\specialrule{.1pt}{.1pt}{.1pt} 
\multirow{2}{*}{$\boldsymbol{L_{GM}}$} & \footnotesize{Best (\%)} & \textbf{59.4} & \textbf{42.2} & \textbf{65.6} & \textbf{71.9} & \textbf{89.1} & \textbf{84.4} & \textbf{87.5} \\
 & \footnotesize{Str.\ best (\%)}      & \textbf{57.8} & \textbf{35.9} & \textbf{60.9} & \textbf{71.9} & \textbf{87.5} & \textbf{84.4} & \textbf{87.5}\\
 \specialrule{1.5pt}{.1pt}{.1pt} 
\end{tabular}
\end{minipage}
\hspace{0.04\linewidth}
\begin{minipage}{.25\linewidth}
\flushright
\tiny
$\textcolor{white}{a}$ \hfill MNIST, $k^+=10$ \hfill $\,$  
	\includegraphics[width=1\linewidth,trim = 00mm 00mm 0mm 00mm, clip]{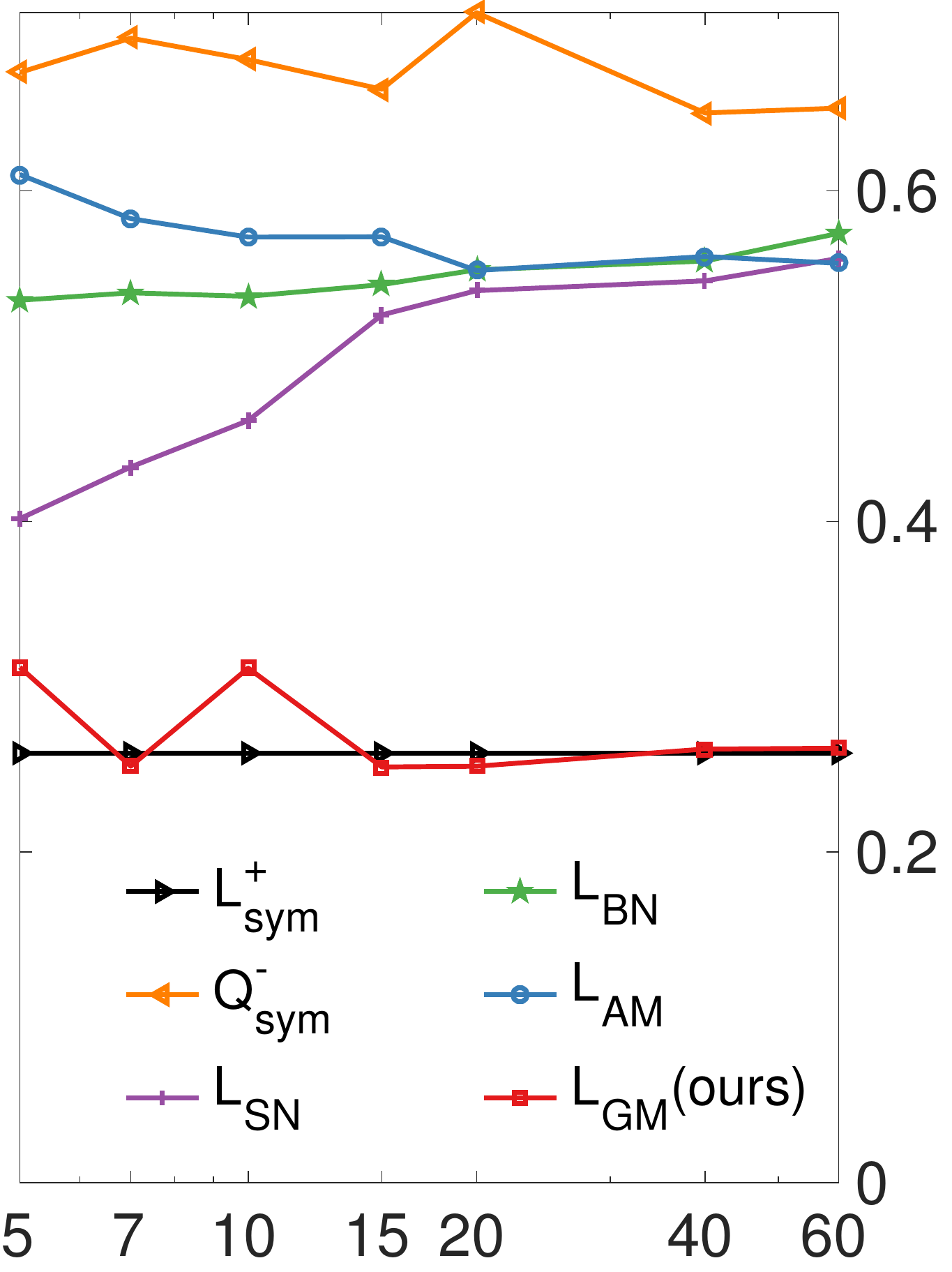}
	 $\textcolor{white}{a}$ \hfill $k^-$ \hfill $\,$
\end{minipage}
\end{minipage}
\vspace{-10pt}
\caption{Experiments on UCI datasets. Left: fraction of cases where methods achieve best and strictly best clustering error. Right: clustering error on MNIST dataset.}
\label{table:UCIexperiments}
\end{table}

In the figure on the right of Table~\ref{table:UCIexperiments} we present the clustering error on MNIST dataset fixing $k^+=10$. 
\old{It is interesting to compare this behavior with the one observed in the SBM.}
\pedronew{With} $Q^-_{\sym}$ \pedronew{one gets} the highest clustering error, which shows that the $k^-$-farthest neighbor graph is a source of noise and 
is not informative. In fact, we observe that a small subset of nodes is the farthest neighborhood of a large fraction of nodes. 
The noise from the $k^-$-farthest neighbor graph is  strongly influencing the performances of $L_{SN}$ and $L_{BN}$, leading to a noisy embedding of the datapoints and in turn to a high clustering error. 
On the other hand we can see that $L_{GM}$ is robust, in the sense that its clustering performances are not affected negatively by the noise \pedronew{in the negative edges}.  Similar behaviors have been observed for the other datasets in Table \ref{table:UCIexperiments}\nosupplementary{, and are shown in supplementary material}. 


\myComment{
\begin{figure}[h]
  \includegraphics[angle=0,width=1\linewidth,clip]{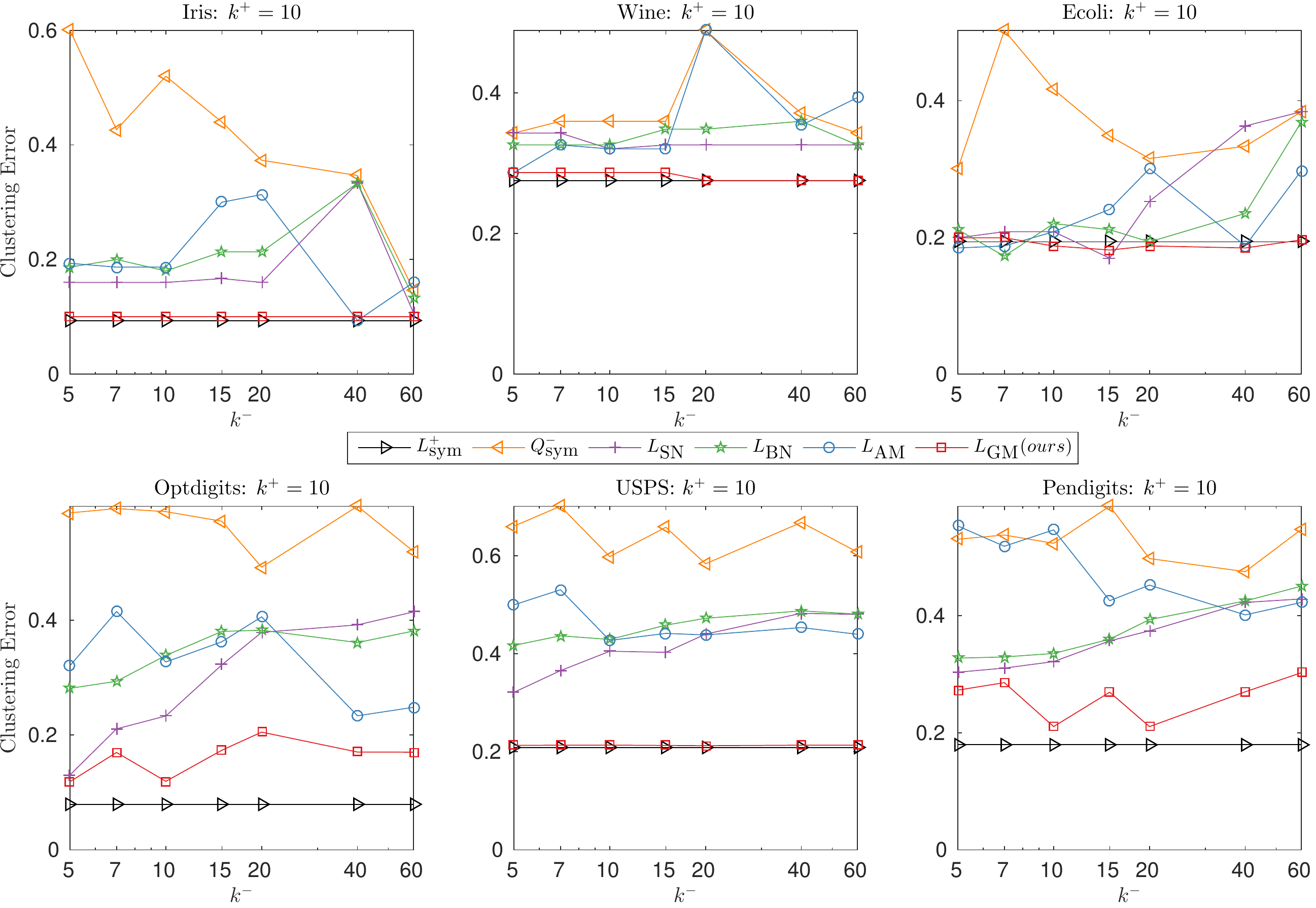}
 \caption{Clustering error on UCI datasets, for $k^+=10$.}
 \label{fig:knnPosParameter10}
\end{figure}
}

\pedro{{\bf Acknowledgments.} The authors acknowledge support by the ERC starting grant NOLEPRO}

\bibliography{ClusteringSignedNetworksWithTheGeometricMeanOfLaplacians}

\begin{thebibliography}{10}

\bibitem{bhatia2009positive}
R.~Bhatia.
\newblock {\em {Positive definite matrices}}.
\newblock Princeton University Press, 2009.

\bibitem{matrixMeanToolbox}
D.~Bini and B.~Ianazzo.
\newblock {The Matrix Means Toolbox}.
\newblock \url{http://bezout.dm.unipi.it/software/mmtoolbox/}, May 2015.

\bibitem{Cartwright:1956:Structural}
D.~Cartwright and F.~Harary.
\newblock {Structural balance: a generalization of Heider's theory.}
\newblock {\em Psychological Review}, 63(5):277--293, 1956.

\bibitem{Chiang:2012:Scalable}
K.~Chiang, J.~Whang, and I.~Dhillon.
\newblock {Scalable clustering of signed networks using balance normalized
  cut}.
\newblock {CIKM}, pages 615--624, 2012.

\bibitem{Davis:1967:Clustering}
J.~A. Davis.
\newblock {Clustering and structural balance in graphs}.
\newblock {\em Human Relations}, 20:181--187, 1967.

\bibitem{Desai:1994:characterization}
M.~Desai and V.~Rao.
\newblock {A characterization of the smallest eigenvalue of a graph}.
\newblock {\em Journal of Graph Theory}, 18(2):181--194, 1994.

\bibitem{Doreian:2009:Partitioning}
P.~Doreian and A.~Mrvar.
\newblock {Partitioning signed social networks}.
\newblock {\em Social Networks}, 31(1):1--11, 2009.

\bibitem{druskin:1998:extendedKrylov}
V.~Druskin and L.~Knizhnerman.
\newblock {Extended Krylov subspaces: approximation of the matrix square root
  and related functions}.
\newblock {\em SIAM J. Matrix Anal. Appl.}, 19:755--771, 1998.

\bibitem{fasi:2016:computing}
M.~Fasi and B.~Iannazzo.
\newblock Computing the weighted geometric mean of two large-scale matrices and
  its inverse times a vector.
\newblock {\em MIMS EPrint: 2016.29}.

\bibitem{Harary:1954:Notion}
F.~Harary.
\newblock {On the notion of balance of a signed graph}.
\newblock {\em Michigan Mathematical Journal}, 2:143--146, 1953.

\bibitem{higham:sign_function}
N.~J. Higham, D.~S. Mackey, N.~Mackey, and F.~Tisseur.
\newblock {Functions preserving matrix groups and iterations for the matrix
  square root}.
\newblock {\em SIAM J. Matrix Anal. Appl.}, 26:849--877, 2005.

\bibitem{Ianazzo:2012:geometricMean}
B.~{Iannazzo}.
\newblock {The geometric mean of two matrices from a computational viewpoint}.
\newblock {\em Numer. Linear Algebra Appl.}, to appear, 2015.

\bibitem{Iannazzo:optimization}
B.~Iannazzo and M.~Porcelli.
\newblock {The Riemannian Barzilai-Borwein method with nonmonotone line-search
  and the Karcher mean computation}.
\newblock {\em Optimization online}, December 2015.

\bibitem{simoncini:2008:error_decay}
L.~Knizhnerman and V.~Simoncini.
\newblock {A new investigation of the extended Krylov subspace method for
  matrix function evaluations}.
\newblock {\em Numer. Linear Algebra Appl.}, 17:615--638, 2009.

\bibitem{Kropipvnik:1996:slovene}
S.~Kropivnik and A.~Mrvar.
\newblock {An Analysis of the Slovene Parliamentary Parties Networks}.
\newblock {\em Development in Statistics and Methodology}, pages 209--216,
  1996.

\bibitem{Kunegis:2010:spectral}
J.~Kunegis, S.~Schmidt, A.~Lommatzsch, J.~Lerner, E.~Luca, and S.~Albayrak.
\newblock {Spectral analysis of signed graphs for clustering, prediction and
  visualization}.
\newblock In {\em {ICDM}}, pages 559--570, 2010.

\bibitem{snapnets}
J.~Leskovec and A.~Krevl.
\newblock {{SNAP Datasets}: {Stanford} Large Network Dataset Collection}.
\newblock \url{http://snap.stanford.edu/data}, June 2014.

\bibitem{Liu2015}
S.~Liu.
\newblock Multi-way dual cheeger constants and spectral bounds of graphs.
\newblock {\em Advances in Mathematics}, 268:306 -- 338, 2015.

\bibitem{Luxburg:2007:tutorial}
U.~Luxburg.
\newblock {A tutorial on spectral clustering}.
\newblock {\em Statistics and Computing}, 17(4):395--416, Dec. 2007.

\bibitem{raissouli:continued_fractions}
M.~Ra\"issouli and F.~Leazizi.
\newblock {Continued fraction expansion of the geometric matrix mean and
  applications.}
\newblock {\em Linear Algebra Appl.}, 359:37--57, 2003.

\bibitem{Read:1954:Cultures}
K.~E. Read.
\newblock {Cultures of the Central Highlands, New Guinea}.
\newblock {\em Southwestern Journal of Anthropology}, 10(1):pp. 1--43, 1954.

\bibitem{rohe2011spectral}
K.~Rohe, S.~Chatterjee, B.~Yu, et~al.
\newblock {Spectral clustering and the high-dimensional stochastic blockmodel}.
\newblock {\em The Annals of Statistics}, 39(4):1878--1915, 2011.

\bibitem{tang2015survey}
J.~Tang, Y.~Chang, C.~Aggarwal, and H.~Liu.
\newblock A survey of signed network mining in social media.
\newblock {\em arXiv preprint arXiv:1511.07569}, 2015.

\end{thebibliography}
\bibliographystyle{abbrv}

\end{document}